\documentclass{article}

\usepackage{arxiv}

\usepackage[utf8]{inputenc} 
\usepackage[T1]{fontenc}    
\usepackage{hyperref}       
\usepackage{url}            
\usepackage{booktabs}       
\usepackage{amsfonts}       
\usepackage{nicefrac}       
\usepackage{microtype}      
\usepackage{lipsum}		
\usepackage{graphicx}
\usepackage[square,numbers]{natbib}
\usepackage{doi}

\usepackage{microtype}
\usepackage{graphicx}
\usepackage{booktabs} 

\usepackage{hyperref}
\usepackage{color, colortbl}
\definecolor{Gray}{gray}{0.9}

\usepackage{tabu, boldline}
\usepackage{graphicx}
\usepackage{subcaption}
\usepackage{hyperref}

\usepackage{algorithm} 
\usepackage{bbm}
\usepackage{algorithmic}
\usepackage{hyperref}
\usepackage{url}
\newtheorem{lemma}{Lemma}  
\newtheorem{lemmaa}{Lemma}  
\newtheorem{theorem}{Theorem}
\newtheorem{theoremm}{Theorem}
\newtheorem{proposition}{Proposition}
\newenvironment{proof}{{\noindent\it Proof.}\quad}{\hfill $\square$\par}
  
\usepackage[utf8]{inputenc} 
\usepackage[T1]{fontenc}    
\usepackage{hyperref}       
\usepackage{url}            
\usepackage{booktabs}       
\usepackage{amsfonts}       
\usepackage{nicefrac}       
\usepackage{microtype}      
\usepackage{graphicx}
\usepackage{xcolor}
\usepackage{upgreek}
\usepackage{amsmath}
\usepackage{bm}
\usepackage{amsfonts}
\usepackage{enumerate}
\usepackage{enumitem}
\usepackage{fancyhdr}
\usepackage{xcolor}
\usepackage{multirow}
\def\*#1{\mathbf{#1}}

\newcommand{\ra}[1]{\renewcommand{\arraystretch}{#1}}
\newcommand{\indep}{\perp \!\!\! \perp}
\usepackage{cleveref}
\usepackage{wrapfig}

\title{On the Impact of Spurious Correlation for Out-of-distribution Detection}

\author{%
  Yifei Ming\\
  Department of Computer Sciences\\
  University of Wisconsin-Madison\\
  \texttt{ming5@wisc.edu} \\
    \And 
  Hang Yin\\
  Department of Computer Sciences\\
  University of Wisconsin-Madison\\
  \texttt{hyin56@wisc.edu} \\
  \And
  Yixuan Li \\
  Department of Computer Sciences\\
  University of Wisconsin-Madison\\
  \texttt{sharonli@cs.wisc.edu} \\
  }

\date{}

\begin{document}

\maketitle
\begin{abstract}

Modern neural networks can assign high confidence to inputs drawn from outside the training distribution, posing threats to models in real-world deployments. While much research attention has been placed on designing new out-of-distribution (OOD) detection methods, the precise definition of OOD is often left in vagueness and falls short of the desired notion of OOD in reality. In this paper, we present a new formalization and model the data shifts by taking into account both the invariant and environmental (spurious) features. Under such formalization, we systematically investigate how spurious correlation in the training set impacts OOD detection. Our results suggest that the detection performance is severely worsened when the correlation between spurious features and labels is increased in the training set.
We further show insights on detection methods that are more effective in reducing the impact of spurious correlation, and provide theoretical analysis on why reliance on environmental features leads to high OOD detection error.  Our work aims to facilitate better understandings of OOD samples and their formalization, as well as the exploration of methods that enhance OOD detection\footnote{Our code is publicly available at \url{https://github.com/deeplearning-wisc/Spurious_OOD}}.




\end{abstract}

\section{Introduction}
\label{intro}

Modern deep neural networks have achieved unprecedented success in known contexts for which they are trained, yet they do not necessarily know what they don’t know~\citep{nguyen2015deep}. 
In particular, neural networks have been shown to produce  high posterior probability for test inputs from out-of-distribution (OOD), which should not be predicted by the model. 
This gives rise to the importance of OOD detection, which aims to identify and handle unknown OOD inputs so that the algorithm can take safety precautions. 

Before we attempt any solution, an important yet often overlooked problem is: {what do we mean by out-of-distribution data}? While the research community lacks a consensus on the precise definition, a common evaluation protocol views data with non-overlapping semantics as OOD inputs~\citep{MSP}. 
For example, an image of a \texttt{cow} can be viewed as an OOD \emph{w.r.t} a model tasked to classify \texttt{cat} vs. \texttt{dog}. However, such an evaluation scheme is often oversimplified and may not capture the nuances and complexity of the problem in reality.


 We begin with a motivating example where a neural network can rely on statistically informative yet \emph{spurious} features in the data. Indeed, many prior works showed that modern neural networks can spuriously rely on the biased features (e.g., background or textures) instead of features of the object to achieve high accuracy~\citep{beery2018recognition, geirhos2018imagenettrained, sagawa2019distributionally}. In Figure~\ref{fig:teaser}, we illustrate a model that exploits the spurious correlation between the \texttt{water background} and label \texttt{waterbird} for prediction. 
 Consequently, a model that relies on spurious features can produce a high-confidence prediction for an OOD input with the same background (\emph{i.e.}, water) but a different semantic label (\emph{e.g.}, boat). This can manifest in downstream OOD detection, yet unexplored in prior works.
 
In this paper, we systematically investigate how spurious correlation in the training set impacts OOD detection. 
We first provide a new formalization and explicitly model the data shifts by taking into account both \textbf{invariant} features and \textbf{environmental} features (Section~\ref{prelim}). Invariant features can be viewed as essential cues directly related to semantic labels, whereas environmental features are non-invariant and can be spurious. Our formalization encapsulates two types of OOD data: 
(1) \emph{spurious OOD}---test samples that contain environmental (non-invariant) features but no invariant features; 
(2) \emph{non-spurious OOD}---inputs that contain neither the environmental nor invariant features, which is more in line with the conventional notion of OOD. We provide an illustration of both types of OOD in Figure~\ref{fig:teaser}. 

\begin{figure*}[t]
  \centering
    \includegraphics[width=0.9\linewidth]{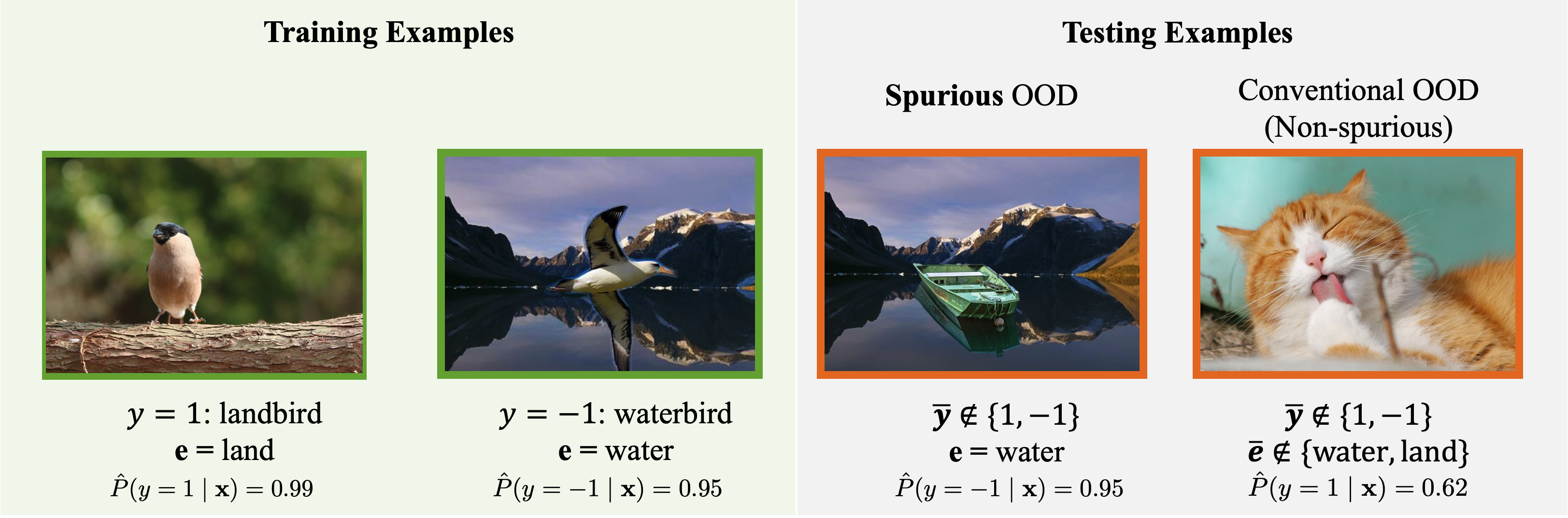}
\caption{\small \textbf{Left} (train): The training examples $\*x$ are generated by a combination of invariant features, dependent on the label $y$; and environmental features, dependent on the environment $e$. In Waterbirds dataset~\citep{sagawa2019distributionally}, $y\in \{\texttt{waterbird}, \texttt{landbird}\}$ is correlated with the environment $e \in \{\text{water}, \text{land}\}$. \textbf{Right} (test): During test time, we consider two types of OOD inputs. Spurious OOD inputs contain the environmental features, but no signals related to the in-distribution classes. Non-spurious OOD inputs have neither environmental features nor invariant features. Confidence scores are computed from a ResNet-18 model trained on Waterbirds~\citep{sagawa2019distributionally}. }
\vspace{-0.3cm}
\label{fig:teaser}
\end{figure*}

Under the new formalization, we conduct extensive experiments and investigate the detection performance under both spurious and non-spurious OOD inputs (Section~\ref{sec:erm}). Our results suggest that spurious correlation in the training data poses a significant challenge to OOD detection. For both spurious and non-spurious OOD samples, the detection performance is severely worsened when the correlation between spurious features and labels is increased in the training set. Further, 
we comprehensively evaluate common OOD detection approaches, and show that feature-based methods have a competitive edge in improving non-spurious OOD detection, while detecting spurious OOD remains challenging (Section~\ref{sec:ood_score}).
To further understand this, we provide theoretical insights on why reliance on non-invariant features leads to high OOD detection error (Section~\ref{sec:theory}). We provably show the existence of spurious OOD inputs with arbitrarily high confidence, which can fail to be distinguished from the ID data. 
Our \textbf{key contributions} are as follows:
\begin{itemize}
    \item We provide a new formalization of OOD detection by explicitly taking into account the separation between invariant features and environmental features. Our formalization encapsulates both spurious and non-spurious OOD. Our work, therefore, provides a complementary perspective in the evaluation of OOD detection. 
     \item We provide systematic investigations on how the extent of spurious correlation in the training set impacts OOD detection. We further show insights on OOD detection solutions that are more effective in mitigating the impact of spurious correlation, with up to 46.73\% reduction of FPR95 in detecting non-spurious OOD data. 
    \item We provide theoretical analysis, provably showing that detecting spurious OOD samples remains challenging due to the model’s reliance on the environmental features.
 
\end{itemize}
Our study provides strong implications for future research on out-of-distribution detection. Our study signifies the importance for future works to evaluate OOD detection algorithms on spurious OOD examples besides standard benchmarks (most of which are non-spurious) to test the limits of the approaches.  We hope that our work will inspire future research on the formalization of the OOD detection problem and algorithmic solutions.

\section{A New Formalization of Out-of-distribution Data}
\label{prelim}
\textbf{Data Model.} We consider a supervised multi-class classification problem, where $\mathcal{X}=\mathbb{R}^d$ denotes the input space and $\mathcal{Y}=\{1,2,...,K\}$ denotes the label space. We assume that the data is drawn from a set of $E$ environments (domains) $\mathcal{E} =\left\{e_{1}, e_{2}, \ldots, e_{E}\right\}$. The inputs $\*x$ is generated by a combination of invariant features $\*z_\text{inv} \in \mathbb{R}^s$, dependent on the label $y$; and environmental feature $\*z_e\in \mathbb{R}^{d_e}$, dependent on the environment $e$:
$$\*x = \tau(\*z_\text{inv}, \*z_e),$$
where 
$\tau$ is a function transformation from the latent features $[\*z_\text{inv},\*z_e]^\top$ to the pixel-space $\mathcal{X}$. The signal $\*z_\text{inv}$ are the cues essential for the recognition of $\*x$ as $y$; examples include the color, the shape of beaks and claws, and fur patterns of birds for classifying  \texttt{waterbird} vs. \texttt{landbird}. Environmental features $\*z_e$, on the other hand, are cues not essential for the recognition but correlated with the target $y$. For example, many waterbird images are taken in water habitat, so water scenes can be considered as $\*z_e$. Under the data model, we have a joint distribution $P(\*x,y, e)$. Each $g=(y,e) \in \mathcal{Y}\times \mathcal{E}$ group has its own distribution over features $[\*z_\text{inv}, \*z_e] \in \mathbb{R}^{s+d_e}$. Furthermore, let $\mathcal{D}_\text{in}^e$ denote the marginal distribution on $\mathcal{X}$ for environment $e$. The union of distributions $\mathcal{D}_\text{in}^e$ over all environments is the in-distribution $\mathcal{D}_\text{in}$. 
\paragraph{Out-of-distribution Data.} In practice, OOD refers to samples from an irrelevant distribution whose label set has no intersection with $\mathcal{Y}$, and therefore should not be predicted by the model. Under our data model, we define data distributional shifts by explicitly taking into account the separation between invariant features and environmental features. Concretely, our formalization encapsulates two types of OOD data defined below. 
\begin{table*}[t]
    \ra{1.2}
    \centering
    \resizebox{0.8\textwidth}{!}{
    \begin{tabular}{cccccccc}\toprule
          & & \multicolumn{2}{c}{$\textbf{r=0.5}$} & \multicolumn{2}{c}{$\textbf{r=0.7}$}
          & \multicolumn{2}{c}{$\textbf{r=0.9}$}
          \\\cmidrule(lr){3-4}\cmidrule(lr){5-6}\cmidrule(lr){7-8} \textbf{OOD Type} & \textbf{Test Set}&
              \textbf{FPR95} $\downarrow$     &      \textbf{AUROC} $\uparrow$   &    \textbf{FPR95}  $\downarrow$     &      \textbf{AUROC}$\uparrow$    &
                  \textbf{FPR95} $\downarrow$      &      \textbf{AUROC} $\uparrow$ \\  \midrule
    \textbf{Spurious OOD} &  & $59.89\pm12.40$ & $88.54\pm4.81$  & $74.22\pm13.12$  & $80.98\pm4.45$ & 
            $74.39\pm12.50$ & $79.81\pm8.43$ \\
            \midrule
            \multirow{4}*{    \textbf{Non-spurious OOD}}
   &  iSUN  & $19.69\pm10.66$ & $91.88\pm4.52$ & $43.22\pm12.50$ & $91.81\pm3.32$ & $57.40\pm15.54$ & $82.45\pm7.98$ \\
    &  LSUN  & $22.60\pm12.08$ & $90.80\pm3.33$ & $43.30\pm16.66$ & $90.09\pm4.51$ & $52.68\pm13.70$ & $84.56\pm8.56$ \\
    &  SVHN  & $15.32\pm5.05$ & $95.71\pm2.20$ & $25.53\pm8.11$ & $95.60\pm2.45$ & $43.89\pm23.80$ & $93.27\pm6.90$ \\
    \bottomrule
    \end{tabular}
    }
    \caption{\small OOD detection performance of  models trained on \textbf{Waterbirds}~\citep{sagawa2019distributionally}. Increased spurious correlation in the training set results in worsen performance for both non-spurious and spurious OOD samples. In particular, spurious OOD is more challenging than non-spurious OOD samples.  Results (mean and std) are estimated over 4 runs for each setting.}
    \label{tab:erm_ablation_waterbird}
     \vspace{-0.3cm}
\end{table*}

\begin{itemize}
\item \textbf{Spurious OOD} is a particularly challenging type of inputs, which contain the \emph{environmental feature, but no invariant feature essential for the label}. Formally, we denote by $\*x = \tau(\*z_{\bar Y}, \*z_e)$, where $\*z_{\bar Y}$ is from an out-of-class label $\bar Y \notin \mathcal{Y}$. For example, this can be seen in Figure~\ref{fig:teaser} (middle right), where the OOD example contains the semantic feature \texttt{boat} $\notin \{\texttt{waterbird}, \texttt{landbird}\}$, yet it has the environmental feature of water background. 
\item \textbf{Non-spurious (conventional) OOD} are inputs that contain \emph{neither the environmental nor the invariant features}, i.e., $\*x = \tau(\*z_{\bar Y}, \*z_{\bar e})$. In particular, $\*z_{\bar Y}$ is sampled from an out-of-class label $\bar Y \notin \mathcal{Y}$, and $\*z_{\bar e}$ is sampled from a different environment $\bar e \notin \mathcal{E}$. For example, an input of an \texttt{indoor cat} falls into this category, where both the semantic label \texttt{cat} and environment \texttt{indoor} are distinct from the in-distribution data of waterbirds and landbirds. 
\end{itemize}

\paragraph{Out-of-distribution Detection.}
OOD detection can be viewed as a binary classification problem. 
Let $f:\mathcal{X} \rightarrow \mathbb{R}^K$ be a neural network trained on samples drawn from the data distribution defined above. 
During inference time, OOD detection can be performed by exercising a thresholding mechanism:
\begin{align}
\label{eq:threshold}
	G_{\lambda}(\*x; f)=\begin{cases} 
      \text{in} & S(\*x;f)\ge \lambda \\
      \text{out} & S(\*x;f) < \lambda 
   \end{cases},
\end{align}
where samples with higher scores $S(\*x;f)$ are classified as ID and vice versa. The threshold $\lambda$ is typically chosen so that a high fraction of ID data (\emph{e.g.,} 95\%) is correctly classified.  
\section{How does spurious correlation impact OOD detection?}
\label{sec:erm}

During training, a classifier may learn to rely on the association between environmental features and labels to make its predictions. Moreover, we hypothesize that such a reliance on environmental features can cause failures in the downstream OOD detection. To verify this, we begin with the most common training objective empirical risk minimization (ERM). 
Given a loss function $\ell$, ERM finds the model $w$ that minimizes the average training loss:
\begin{align}
    \mathcal{\hat R}(w) = \mathbb{E}_{(\*x,y,e)\sim\hat P} [\ell (w; (\*x,y,e))].
\end{align}

We now describe the datasets we use for model training and OOD detection tasks. We consider three tasks that are commonly used in the literature. We start with a natural image dataset Waterbirds, and then move onto the CelebA dataset~\citep{liu2015faceattributes}. Due to space constraints, a third evaluation task on ColorMNIST is in the Supplementary.

\paragraph{Evaluation Task 1: Waterbirds.} Introduced in \citep{sagawa2019distributionally}, this dataset is used to explore the spurious correlation between the image background and bird types, specifically $\mathcal{E} \in \{\texttt{water}, \texttt{land}\}$ and $\mathcal{Y} \in \{\texttt{waterbirds}, \texttt{landbirds}\}$. We also control the correlation between $y$ and $e$ during training as $ r \in \{0.5, 0.7, 0.9\}$. The correlation $r$ is defined as $r = P(e = \texttt{water} \mid y = \texttt{waterbirds} ) = P(e = \texttt{land} \mid y = \texttt{landbirds})$. For spurious OOD, we adopt a subset of images of land and water from the Places dataset \citep{zhou2017places}. For non-spurious OOD, we follow the common practice and use the  \texttt{SVHN}~\citep{svhn}, \texttt{LSUN}~\citep{lsun}, and \texttt{iSUN}~\citep{xu2015turkergaze} datasets. 

\paragraph{Evaluation Task 2: CelebA.} In order to further validate our findings beyond background spurious (environmental) features, we also evaluate on the CelebA~\citep{liu2015faceattributes} dataset.  The classifier is trained to differentiate the hair color (grey vs. non-grey) with $\mathcal{Y} = \{\texttt{grey hair}, \texttt{nongrey hair}\}$. The environments $\mathcal{E} = \{\texttt{male}, \texttt{female}\}$ denote the gender of the person. In the training set, ``Grey hair'' is highly correlated with ``Male'', where $82.9\%$ ($r\approx0.8$) images with grey hair are male. Spurious OOD inputs consist of \texttt{bald male}, which contain environmental features (gender) without invariant features (hair). The non-spurious OOD test suite is the same as above (\texttt{SVHN}, \texttt{LSUN}, and \texttt{iSUN}). Figure~\ref{fig:celeba} illustates ID samples, spurious and non-spurious OOD test sets. We also subsample the dataset to ablate the effect of $r$; see results are in the Supplementary.


\begin{figure*}[t]
  \centering
  \vspace{-0.3cm}
    \includegraphics[width=0.85\linewidth]{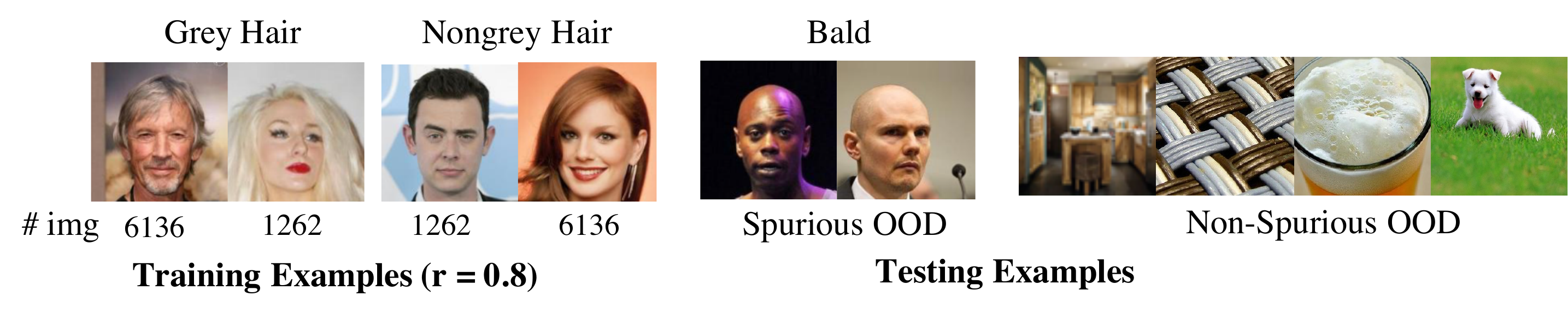}
      \vspace{-10pt}
\caption{\small For CelebA, the classifier is trained to differentiate the hair color (grey vs. non-grey). \textbf{Left}: Training environments.  $82.9\%$ images with grey hair are male, whereas $82.9\%$ images with non-grey hair are female. \textbf{Middle}: Spurious OOD inputs contain the environmental feature (male) without invariant features (hair). \textbf{Right}: Non-spurious OOD samples consist of images with diverse semantics without human faces.}
\label{fig:celeba}
 \vspace{-0.3cm}
\end{figure*}
\paragraph{Results and Insights.} We train on ResNet-18~\citep{he2016deep} for both tasks. See Appendix for details on hyperparameters and in-distribution performance. We summarize the OOD detection performance in Table~\ref{tab:erm_ablation_waterbird} (Waterbirds), Table~\ref{tab:erm_celeba} (CelebA) and Table~\ref{tab:erm_ablation_cmnist} (ColorMNIST). 

There are several salient observations. 
\textbf{First}, for both spurious and non-spurious OOD samples, the detection performance is severely worsened when the correlation between spurious features and labels is increased in the training set. Take the Waterbirds task as an example, under correlation $r=0.5$, the average false positive rate (FPR95) for spurious OOD samples is $59.89\%$, and increases to $74.39$\% when $r=0.9$. Similar trends also hold for other datasets. \textbf{Second}, spurious OOD is much more challenging to be detected compared to non-spurious OOD. From Table\ref{tab:erm_ablation_waterbird}, under correlation $r =0.7$, the average FPR95 is $37.35\%$ for non-spurious OOD, and increases to $74.22\%$ for spurious OOD. Similar observations hold under different correlation and different training datasets. \textbf{Third}, for non-spurious OOD, samples that are more semantically dissimilar to ID are easier to detect. Take Waterbirds as an example, images containing scenes (e.g. LSUN and iSUN) are more similar to the training samples compared to images of numbers (e.g. SVHN), resulting in higher FPR95 (e.g. $43.22\%$ for iSUN compared to $25.53\%$ for SVHN under $r=0.7$). 

\begin{table}[h]
    \ra{1.2}
    \centering
    \resizebox{0.45\textwidth}{!}{
    \begin{tabular}{cccc}\toprule
          \textbf{OOD Type} & \textbf{Test Set}&
              \textbf{FPR95} $\downarrow$     &      \textbf{AUROC} $\uparrow$  \\  \midrule
    \textbf{Spurious OOD} & & $71.28\pm4.12$ & $82.04\pm2.64$  \\
            \midrule
            \multirow{4}*{\textbf{Non-spurious OOD}}
   &  iSUN  & $17.35\pm2.97$ &  $97.03\pm0.30$  \\
    &  LSUN  & $18.85\pm2.44$ & $96.90\pm0.17$   \\
    &  SVHN  & $5.63\pm2.60$ &  $98.64\pm0.21$  \\
    \bottomrule
    \end{tabular}
    }
    \caption{\small OOD detection performance of  models trained on \textbf{CelebA}~\citep{liu2015faceattributes} with $r\approx0.8$. Spurious OOD test data incurs much higher FPR than non-spurious OOD data. Results (mean and std) are estimated over 4 runs for each setting.}
        \vspace{-0.4cm}
    \label{tab:erm_celeba}
\end{table}

Our results suggest that spurious correlation poses a significant threat to the model. In particular, a model can produce high-confidence predictions on the spurious OOD, due to the reliance on the environmental feature (\emph{e.g.}, background information) rather than the invariant feature (\emph{e.g.}, bird species). To verify that the spurious feature causes poor detection performance, we show that the classifier frequently predicts the spurious OOD as the ID class with the same environmental feature. For Waterbirds, on average $93.9\%$ of OOD samples with water background is classified as waterbirds, and $80.7\%$ of OOD samples with land background is classified as land birds. For the CelebA dataset, on average $86.5\%$ of spurious OOD samples (bold male) are classified as grey hair. 
Note that our results here are based on the energy score~\citep{liu2020energy}, which is one competitive detection method derived from the model output (logits) and has shown superior OOD detection performance over directly using the predictive confidence score. Next, we provide an expansive evaluation using a broader suite of OOD scoring functions in  Section~\ref{sec:ood_score}.

\begin{table*}[ht]
    \ra{1.2}
    \centering
    \resizebox{\textwidth}{!}{
    \begin{tabular}{lcccccccccccccccccccc}\toprule
    \textbf{Scoring Func}
        & \multicolumn{4}{c}{MSP \citep{MSP}} & \multicolumn{4}{c}{ODIN \citep{liang2018enhancing}}
          & \multicolumn{4}{c}{Mahalanobis \citep{Maha}} & \multicolumn{4}{c}{Energy \citep{liu2020energy}} & \multicolumn{4}{c}{Gram \citep{gram}}
          \\\cmidrule(lr){2-5}\cmidrule(lr){6-9}\cmidrule(lr){10-13}\cmidrule(lr){14-17} \cmidrule(lr){18-21}
    \textbf{Metric}
        & \multicolumn{2}{c}{\textbf{FPR95}$\downarrow$} &\multicolumn{2}{c}{\textbf{AUROC}$\uparrow$}
        & \multicolumn{2}{c}{\textbf{FPR95}$\downarrow$} &\multicolumn{2}{c}{\textbf{AUROC}$\uparrow$}
        & \multicolumn{2}{c}{\textbf{FPR95}$\downarrow$} &\multicolumn{2}{c}{\textbf{AUROC}$\uparrow$}
        & \multicolumn{2}{c}{\textbf{FPR95}$\downarrow$} &\multicolumn{2}{c}{\textbf{AUROC}$\uparrow$}
        & \multicolumn{2}{c}{\textbf{FPR95}$\downarrow$} &\multicolumn{2}{c}{\textbf{AUROC}$\uparrow$}\\
        \cmidrule(lr){2-3}\cmidrule(lr){4-5}\cmidrule(lr){6-7}\cmidrule(lr){8-9}\cmidrule(lr){10-11}\cmidrule(lr){12-13}\cmidrule(lr){14-15}\cmidrule(lr){16-17} \cmidrule(lr){18-19} \cmidrule(lr){20-21}
    \textbf{In-distribution Data}
        & \textbf{SP} & \textbf{NSP} & \textbf{SP} & \textbf{NSP} & \textbf{SP} & \textbf{NSP} & \textbf{SP} & \textbf{NSP} & \textbf{SP} & \textbf{NSP} & \textbf{SP} & \textbf{NSP} & \textbf{SP} & \textbf{NSP} & \textbf{SP} & \textbf{NSP}& \textbf{SP} & \textbf{NSP}& \textbf{SP} & \textbf{NSP}\\ 
        \midrule
    \textbf{ColorMNIST} & \cellcolor{gray!25} 42.99 & \textcolor{red}{3.15}  &\cellcolor{gray!25} 77.75&99.13 &\cellcolor{gray!25}38.06 &1.88 &\cellcolor{gray!25} 78.78 &99.01  & \cellcolor{gray!25}14.97& \textcolor{blue}{0.04} &\cellcolor{gray!25}88.65 & 99.54& \cellcolor{gray!25}30.45&7.65  &\cellcolor{gray!25}86.74&97.54 &\cellcolor{gray!25} 4.33 &0.05 & \cellcolor{gray!25}96.89 &99.40  \\
        \textbf{Waterbirds } &74.68 \cellcolor{gray!25} &\textcolor{red}{47.53}  &\cellcolor{gray!25} 79.22&92.34 &\cellcolor{gray!25} 77.25 & 34.06 &\cellcolor{gray!25} 81.04 & 93.48 & \cellcolor{gray!25} 69.35&  \textcolor{blue}{0.80} &\cellcolor{gray!25} 82.73& 99.51 &\cellcolor{gray!25}74.22 &37.35 &\cellcolor{gray!25}80.98&92.50    &\cellcolor{gray!25} 58.25 & 0.65 &\cellcolor{gray!25} 87.33& 99.71\\
         \textbf{CelebA} & \cellcolor{gray!25}83.70 &\textcolor{red}{22.60}  &\cellcolor{gray!25} 68.22& 90.21 &\cellcolor{gray!25}81.07 &11.49 &\cellcolor{gray!25} 75.22 &  89.11 & \cellcolor{gray!25}78.75 &  \textcolor{blue}{2.33} &\cellcolor{gray!25}83.12& 98.93&\cellcolor{gray!25}71.28 &13.94 &\cellcolor{gray!25}82.04& 97.51 &\cellcolor{gray!25} 81.21& 3.11 &\cellcolor{gray!25} 68.58 &98.96 \\
        
            
    \bottomrule
    \end{tabular}
    }
    \caption{\small Performance for different post-hoc OOD detection methods when the spurious correlation is high in the training set. We choose $r=0.45$ for ColorMNIST, $r=0.7$ for Waterbirds, and $r=0.8$ for CelebA. SP stands for Spurious OOD test set. NSP denotes non-spurious OOD, where the results are averaged over 3 OOD test sets (see details in Section~\ref{sec:erm}).}
    \label{tab:diff_ood_score}
    \end{table*}

\section{How to reduce the impact of spurious correlation for OOD detection?}
\label{sec:ood_score}
The results in the previous section naturally prompt the question: how can we better detect spurious and non-spurious OOD inputs when the training dataset contains spurious correlation?  In this section, we comprehensively evaluate common OOD detection approaches, 
and show that feature-based methods have a competitive edge in improving non-spurious OOD detection, while detecting spurious OOD remains challenging (which we further explain theoretically in Section~\ref{sec:theory}). 

\paragraph{Feature-based vs. Output-based OOD Detection.}  Section~\ref{sec:erm} suggests that OOD detection becomes challenging for output-based methods especially when the training set contains high spurious correlation.
However, the efficacy of using representation space for OOD detection remains unknown.  In this section, we consider a suite of common scoring functions including maximum softmax probability (MSP)~\citep{MSP}, ODIN score~\citep{liang2018enhancing, GODIN}, Mahalanobis distance-based score~\citep{Maha}, energy score~\citep{liu2020energy}, and Gram matrix-based score~\citep{gram}---all of which can be derived \emph{post hoc}\footnote{Note that Generalized-ODIN requires modifying the training objective and model retraining. For fairness, we primarily consider strict post-hoc methods based on the standard cross-entropy loss.} from a trained model. Among those, Mahalanobis and Gram Matrices can be viewed as feature-based methods. For example, \citet{Maha} estimates class-conditional Gaussian distributions in the representation space and then uses the maximum Mahalanobis distance as the OOD scoring function. Data points that are sufficiently far away from all the class centroids are more likely to be OOD.  

\paragraph{Results.} The performance comparison is shown in Table~\ref{tab:diff_ood_score}. Several interesting observations can be drawn. \textbf{First}, we can observe a significant performance gap between \emph{spurious OOD} (SP) and \emph{non-spurious OOD} (NSP), irrespective of the OOD scoring function in use. This observation is in line with our findings in Section~\ref{sec:erm}. 
\textbf{Second}, the OOD detection performance is generally improved with the feature-based scoring functions such as Mahalanobis distance score~\citep{Maha} and Gram Matrix score~\citep{gram}, compared to scoring functions based on the output space (\emph{e.g.}, MSP, ODIN, and energy). The improvement is substantial for non-spurious OOD data. For example, on Waterbirds, FPR95 is reduced by 46.73\% with Mahalanobis score compared to using MSP score. For spurious OOD data, the performance improvement is most pronounced using the Mahalanobis score. Noticeably, using the Mahalanobis score, the FPR95 is reduced by 28.02\% on the ColorMNIST dataset, compared to using the MSP score. 
Our results suggest that feature space preserves useful information that can more effectively distinguish between ID and OOD data.
\begin{figure*}[htb]
  \begin{subfigure}[m]{0.68\linewidth}
    \includegraphics[width=\linewidth]{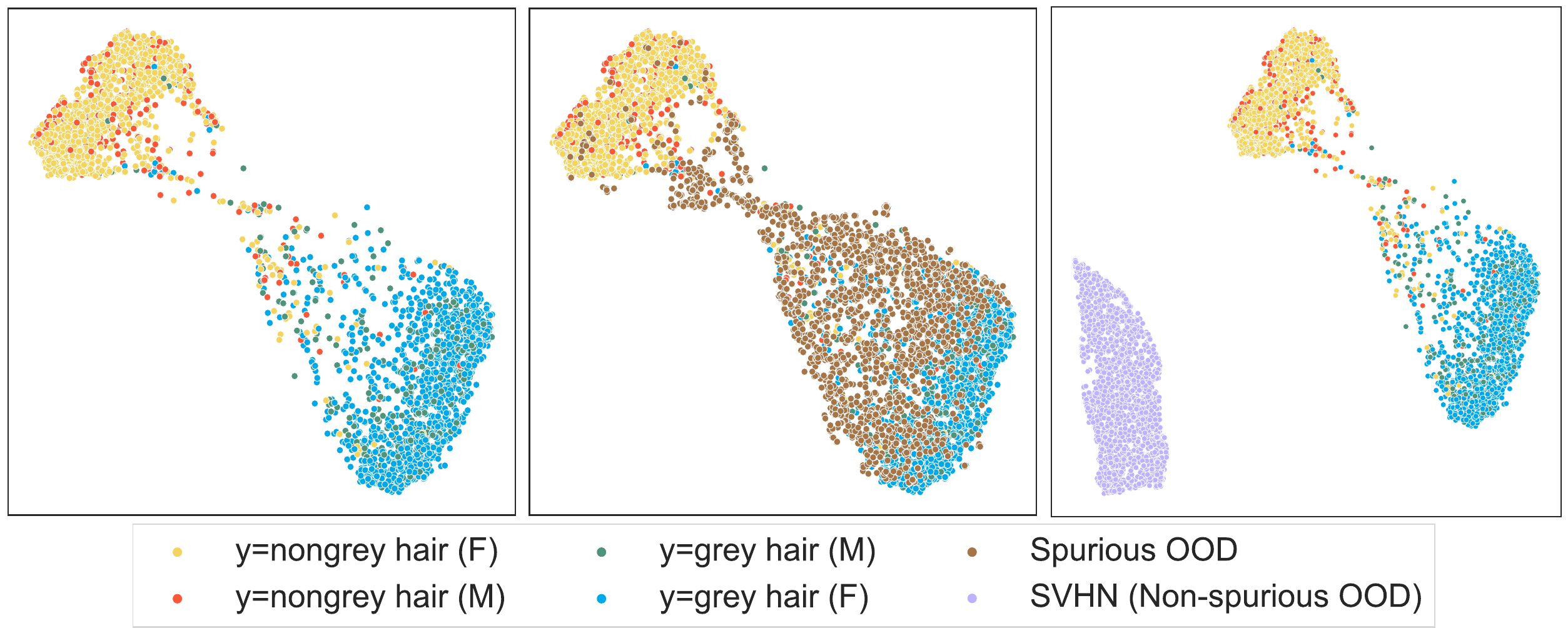}
     \caption{Visualization of Feature Embedding}
     \label{fig:umap_feature}
  \end{subfigure}
    \begin{subfigure}[m]{0.31\linewidth}
    \includegraphics[width=\linewidth]{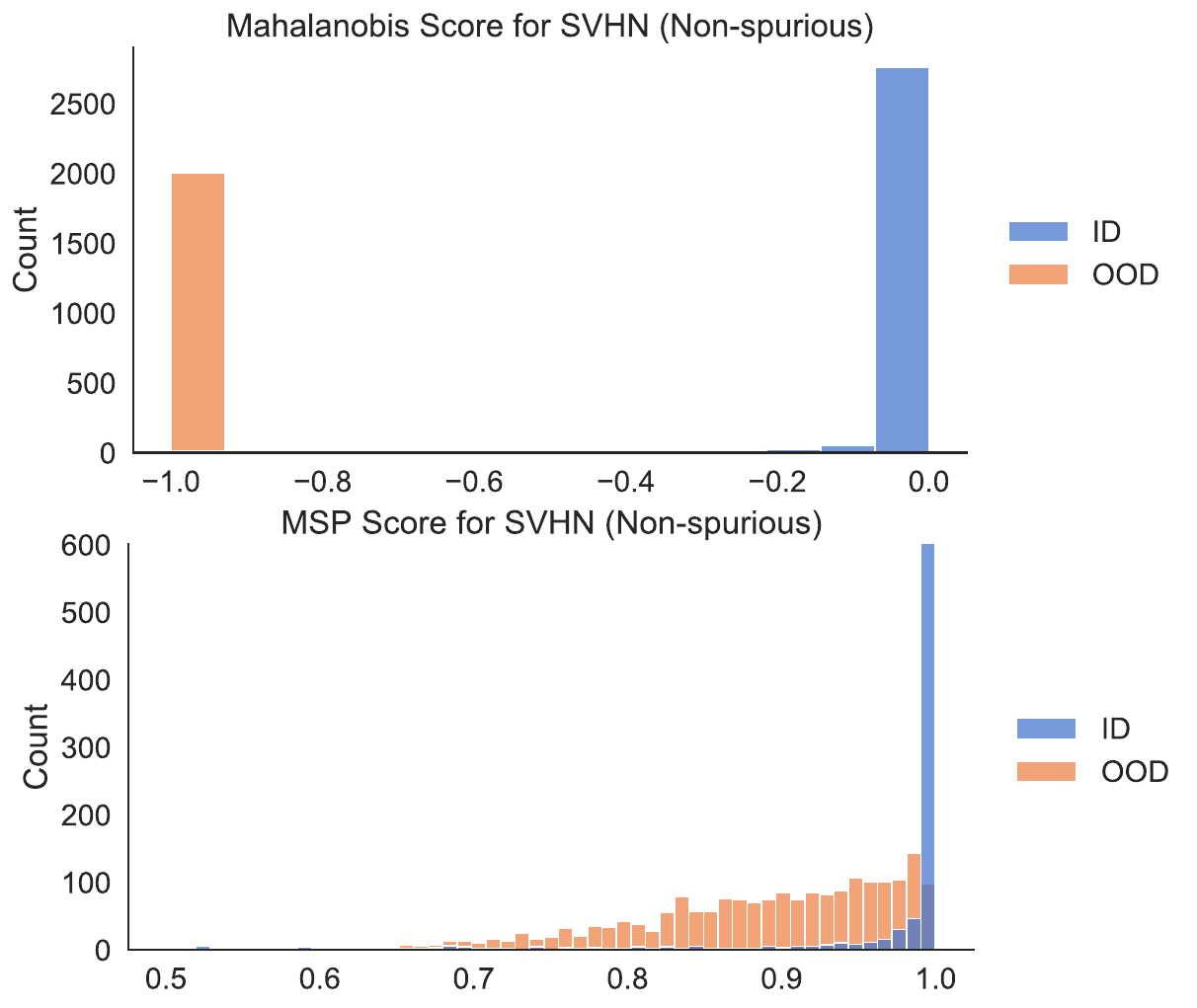}
    \caption{Histograms of Different Scores}
    \label{fig:histogram}
  \end{subfigure}
\caption{ \small (a) \textbf{Left}: Feature for in-distribution data only. (a) \textbf{Middle}: Feature for both ID and spurious OOD data. (a) \textbf{Right}: Feature for ID and non-spurious OOD data (SVHN). M and F in parentheses stand for male and female respectively.  (b) Histogram of Mahalanobis score and MSP score for ID and SVHN (Non-spurious OOD). Full results for other non-spurious OOD datasets (iSUN and LSUN) are in the Supplementary.} 
\vspace{-0.3cm}
\end{figure*}

\paragraph{Analysis and Visualizations.} To provide further insights on why the feature-based method is more desirable, we show the visualization of embeddings in Figure~\ref{fig:umap_feature}. The visualization is based on the CelebA task. 
From Figure~\ref{fig:umap_feature} (left), we observe a clear separation between the two class labels. Within each class label, data points from both environments are well mixed (e.g., see the green and blue dots). In Figure~\ref{fig:umap_feature} (middle), we visualize the embedding of ID data together with spurious OOD inputs, which contain the environmental feature (\texttt{male}). Spurious OOD (bold male) lies between the two ID clusters, with some portion overlapping with the ID samples, signifying
the hardness of this type of OOD.
This is in stark contrast with non-spurious OOD inputs shown in Figure~\ref{fig:umap_feature} (right), where a clear separation between ID and OOD (purple) can be observed. This shows that feature space contains useful information that can be leveraged for OOD detection, especially for conventional non-spurious OOD inputs. Moreover, by comparing the histogram of Mahalanobis distance (top) and MSP score (bottom) in Figure~\ref{fig:histogram}, we can further verify that ID and OOD data is much more separable with the Mahalanobis distance. Therefore, our results suggest that feature-based methods show promise for improving non-spurious OOD detection when the training set contains spurious correlation, while there still exists large room for improvement on spurious OOD detection.




\section{Why is it hard to detect spurious OOD?}
\label{sec:theory}
Given the results above, a natural question arises: why is it hard to detect spurious OOD inputs? To better understand this issue, we now provide theoretical insights. In what follows, we first model the ID and OOD data distributions and then derive mathematically the model output of invariant classifier, where the model aims not to rely on the environmental features for prediction. 


\paragraph{Setup.} We consider a binary classification task where $y \in\{-1, 1\}$, and is drawn according to a fixed probability $\eta:= P(y=1)$.
We assume both the invariant features $\*z_\text{inv}$ and environmental features $\*z_e$ are drawn from Gaussian distributions: $$\*z_\text{inv} \sim \mathcal{N}\left(y \cdot \bm\mu_\text{inv}, \sigma_\text{inv}^{2} I\right), \quad \*z_e \sim \mathcal{N}\left(y \cdot \bm\mu_{e}, \sigma_{e}^{2} I\right)$$ 
where $\bm\mu_{e} \in \mathbb{R}^{d_e}$, $\bm\mu_\text{inv} \in \mathbb{R}^{s}$, and $I$ is the identity matrix. Note that the parameters $\bm\mu_\text{inv}$ and $\sigma^2_\text{inv}$ are the same for all environments. In contrast, the environmental parameters $\bm\mu_\text{e}$ and $\sigma^2_e$ are different across $e$, where the subscript is used to indicate the dependence on the environment and the index of the environment. In what follows, we present the results, with detailed proof deferred in the Appendix. 

\begin{lemma} 
\label{lem:bayes}
(Bayes optimal classifier) For any feature vector which is a linear combination of the invariant and environmental features $\Phi_e(\*x) = M_\text{inv}\*z_\text{inv} + M_e\*z_e$, the optimal linear classifier for an environment $e$ has the corresponding coefficient $ 2\Sigma_\Phi^{-1}\bm\mu_{\Phi}$, where:
\begin{align*}
    \bm\mu_{\Phi}&  = M_\text{inv} \bm\mu_{\text{inv}}+M_e \bm\mu_{e}\\
    \Sigma_{\Phi}& = M_\text{inv} M_\text{inv}^{T} \sigma_{\text{inv}}^{2}+M_e M_e^{T} \sigma_{e}^{2}
\end{align*}
\label{lemma:bayes}
\vspace{-15pt}
\end{lemma}
Note that the Bayes optimal classifier uses environmental features which are informative of the label but non-invariant. Rather, we hope to rely \emph{only} on invariant features while ignoring environmental features. Such a predictor is also referred to as \emph{optimal invariant predictor}~\citep{rosenfeld2020risks}, which is specified in the following. Note that this is a special case of Lemma~\ref{lemma:bayes} with $M_\text{inv} = I$ and $M_e = {0}$.
\begin{proposition} 
\label{prop:invariant}
(Optimal invariant classifier using invariant features) Assume the featurizer recovers the invariant feature $\Phi_e(\*x) = [\*z_\text{inv}] \; \forall e \in \mathcal{E}$, the optimal invariant classifier has the corresponding coefficient $2\bm\mu_\text{inv}/\sigma^2_\text{inv}$.\footnote{The constant term in the classifier weights is $\log \eta / (1-\eta)$, which we omit here and in the sequel.}


\end{proposition}

The optimal invariant classifier explicitly ignores the environmental features. However, an invariant classifier learned does not necessarily depend only on the invariant features. Next Lemma shows that \emph{it can be possible to learn an invariant classifier that relies on the environmental features while achieving lower risk than the optimal invariant classifier}.
\begin{lemma} (Invariant classifier using non-invariant features) 
\label{lem:inv}Suppose $E \leq d_{e}$, given a set of environments $\mathcal{E}=\left\{e_{1}, e_{2}, \ldots, e_{E}\right\}$ such that all environmental means are linearly independent. Then there always exists a unit-norm vector $\*p$ and positive fixed scalar $\beta$ such that $\beta = \*p^{T} \bm\mu_{e}/ \sigma_{e}^{2}$ $\forall e \in \mathcal{E}$. The resulting optimal classifier weights are 
$$\hat w = \left[ \begin{array}{c} \beta_{\text{inv}} \\ 2\beta \end{array}\right] = \left[ \begin{array}{c} 2\bm\mu_{\text{inv}} / \sigma_\text{inv}^{2}\\ 2\*p^\top\bm\mu_e / \sigma_e^2 \end{array}\right].$$
\label{short_cut}
\vspace{-9pt}
\end{lemma}

Note that the optimal classifier weight $2\beta$ is a constant, which does not depend on the environment (and neither does the optimal coefficient for $\*z_\text{inv}$). The projection vector $\*p$ acts as a "short-cut" that the learner can use to yield an insidious surrogate signal $\*p^\top\*z_{e}$. Similar to $\*z_\text{inv}$, this insidious signal can also lead to an invariant predictor (across environments) admissible by invariant learning methods. In other words, despite the varying data distribution across environments, the optimal classifier (using non-invariant features) is the same for each environment. We now show our main results, where OOD detection can fail under such an invariant classifier. 

\begin{theorem}\label{lem:thm} (Failure of OOD detection under invariant classifier) Consider an out-of-distribution input which contains the environmental feature: $\Phi_{\text{out}}(\*x) = M_{\text{inv}}\*z_{\text{out}} + M_e\*z_e$, where  $\*z_{\text{out}} \perp \bm\mu_{\text{inv}}$. Given the invariant classifier (cf. Lemma 2), the posterior probability for the OOD input is $p(y = 1 \mid \Phi_{\text{out}}) = \sigma\left(2\*p^\top\*z_e\beta + \log \eta / (1-\eta) \right)$, where $\sigma$ is the logistic function. Thus for arbitrary confidence $0<c := P(y = 1 \mid \Phi_{\text{out}})<1$, there exists $\Phi_{\text{out}}(\*x)$ with $\*z_e$ such that $\*p^\top\*z_e = \frac{1}{2\beta} \log \frac{c(1 -\eta)}{\eta(1-c)}$. 

\end{theorem}
Our theorem above signifies the existence of OOD inputs that can trigger high-confidence predictions on in-distribution classes yet contain no meaningful feature related to the labels in $\mathcal{Y}=\{1,-1\}$ at all. An OOD detector can fail to detect these inputs with predictions that are indistinguishable from ID data. We provide a simple toy example to explain this phenomenon further. 

\begin{figure*}[t]
\centering
  \begin{subfigure}[b]{0.37\linewidth}
    \includegraphics[width=\linewidth]{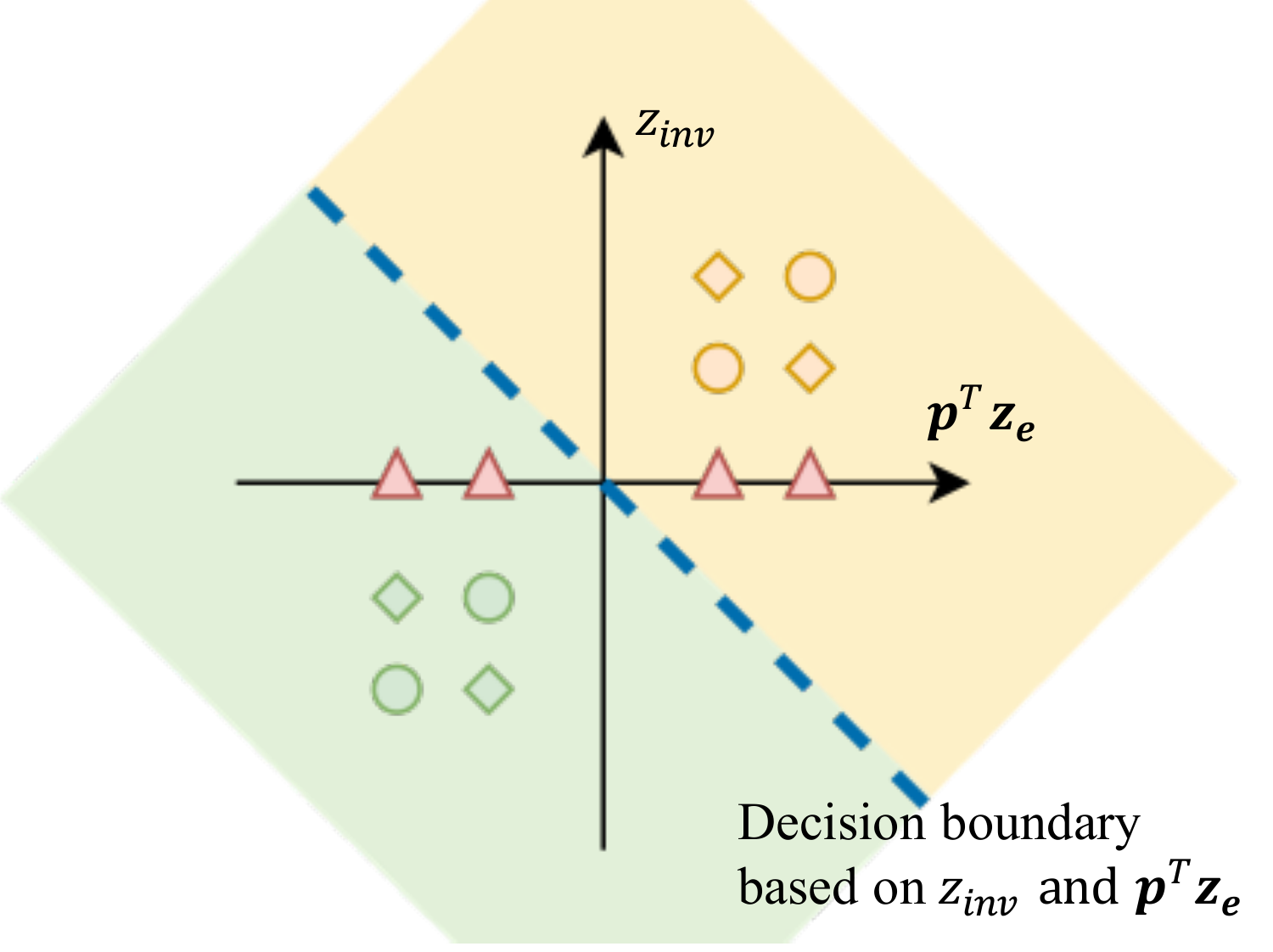}
     \caption{Feature based on $\*z_\text{inv}$ and $\*p^\top\*z_e$ }
     \label{fig:spurious_decision}
  \end{subfigure}
  \hspace{7em}
    \begin{subfigure}[b]{0.37\linewidth}
    \includegraphics[width=\linewidth]{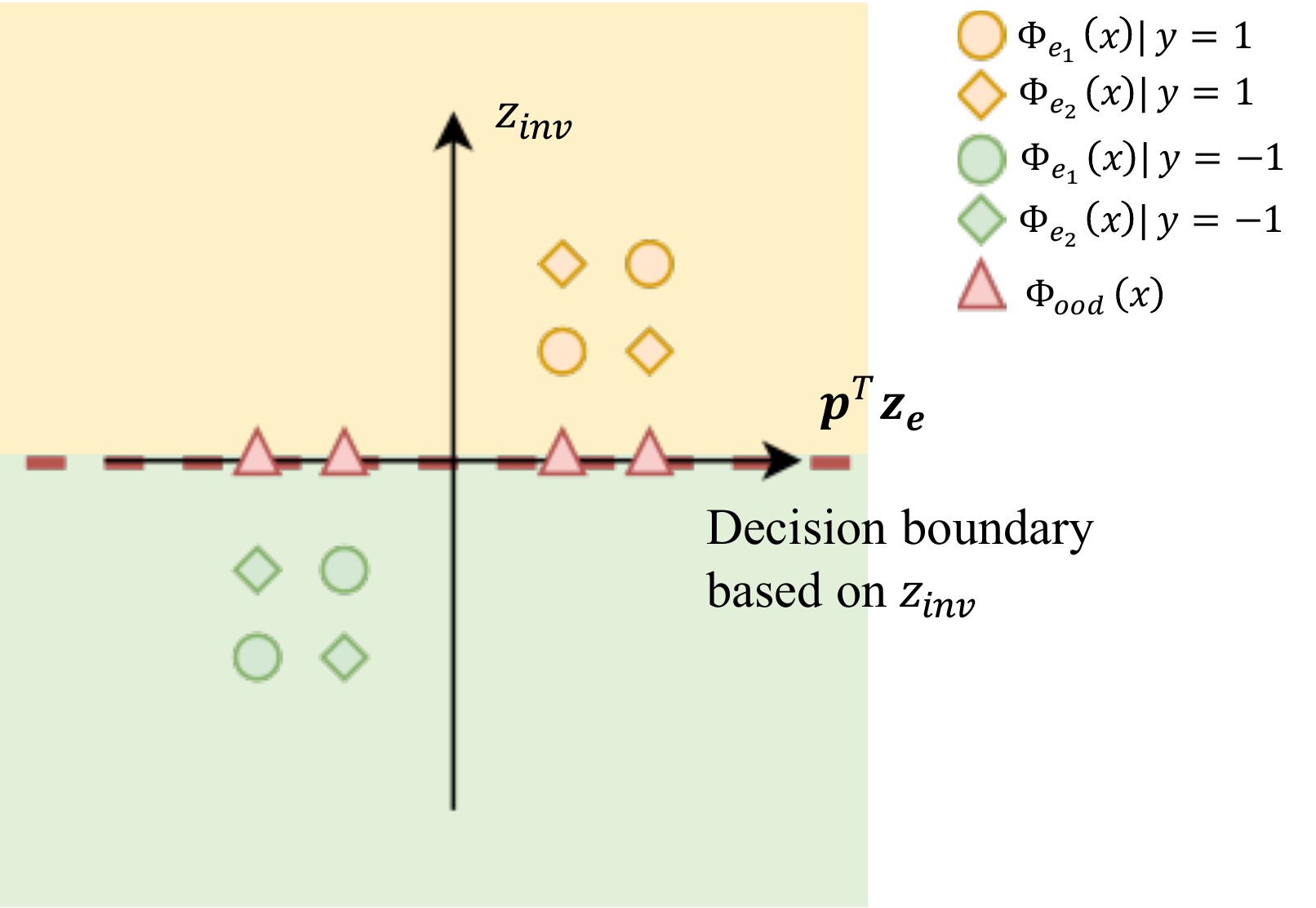}
    \caption{Feature based on $\*z_\text{inv}$ }
    \label{fig:ideal_invariant}
  \end{subfigure}
\caption{ \small The ID data is comprised of two classes $y=1$ (yellow) and $y=-1$ (green). Two environments are shown as circle and diamond, respectively. (a) The invariant decision boundary (blue dashed line) is based on both the invariant feature $z_\text{inv}$ and environmental features $z_e$. Illustration of the existence of OOD inputs (red triangles) that can be predicted as in-distribution with high confidence, therefore can fail to be detected by OOD methods (e.g., using predictive confidence threshold). (b) An ideal case when the invariant decision boundary is purely based on $z_\text{inv}$ (red dashed line). The OOD inputs lie on the decision boundary and will be predicted as $y=1$ or $y=-1$ with a probability 0.5.} 
\vspace{-0.3cm}
\end{figure*}

\paragraph{An Intuitive Example.} An illustrative example with two environments is provided in Figure~\ref{fig:spurious_decision}. The feature representations for examples in environments 1 and 2 are shown as circle and diamond, respectively. In-distribution samples with different colors correspond to different labels: yellow indicates $y=1$ and green indicates $y=-1$. The decision boundary of classification is denoted by the dashed line, which relies on both the invariant features $\*z_\text{inv}$ and environmental features $\*z_e$.  It can be seen that if the feature representation relies on environmental features $\*p^\top\*z_e$, spurious OOD samples (red triangles) can trick the classifier into recognizing OOD samples as one of the in-distribution classes with high confidence, posing severe threats to OOD detection. 

In contrast, under an ideal case when the invariant classifier only uses invariant features $\*z_\text{inv}$, the optimal decision boundary is a horizontal dashed line (see Figure~\ref{fig:ideal_invariant}). OOD inputs (red triangles) will be predicted with a probability of 0.5 since they lie on the decision boundary. 

\paragraph{Remark.} As a special case, if the representation consists purely of environmental features, i.e.,   $\Phi_e(\*x) = \left[\*z_e\right]$, the resulting optimal classifier weights are $2\*p^\top\*\mu_e / \sigma_e^2 = 2\beta$, a fixed scalar that is still invariant across environments. Lemma~\ref{lemma:pure} below shows that such a predictor can yield low risks under certain conditions. Our main theorem above still holds under such a predictor. 
\begin{lemma}
\label{lemma:pure}
(Existence of purely environmental predictors with low risks ~\citep{rosenfeld2020risks}) There exists a representation constructed purely relying on environmental features based on the short-cut direction $\*p$  that achieves lower risk than the optimal invariant predictor on every environment e such that $\sigma_{e} \beta>\sigma_{\text{inv}}^{-1}\left\|\mu_{\text{inv}}\right\|_{2}$ and $2 \sigma_{e} \beta \sigma_{\text{inv}}^{-1}\left\|\mu_{\text{inv}}\right\|_{2} \geq\left|\log \eta / (1-\eta)\right|$.
\end{lemma}

\paragraph{Summary.} To summarize, the theoretical analysis demonstrates the difficulty of recovering the invariant classifier without using environmental features. In particular, there exists an invariant classifier that uses non-invariant features, and achieves lower risks than the classifiers only based on invariant features. 
As a result, spurious OOD samples can utilize environmental clues to deteriorate the OOD detection performance. Our main theorem provably shows the existence of OOD inputs with arbitrarily high confidence, and can fail to be distinguished from the ID data. 

\paragraph{Extension: Empirical Validation of Theoretical Analysis.} To further validate our analysis above, we comprehensively evaluate the OOD detection performance of models that are trained with recent prominent domain invariance learning objectives~\citep{arjovsky2019invariant,  bahng2020learning, krueger2020out, ganin2016domain, li2018deep, sagawa2019distributionally}~(Section~\ref{sec:irm} in Appendix). The results align with our theoretical analysis. 

\section{Discussion and related works}
\label{sec:related}
\paragraph{Out-of-distribution Uncertainty Estimation.}  
The phenomenon of neural networks' overconfidence to out-of-distribution data is revealed by Nguyen \emph{et al.}~\cite{nguyen2015deep}.
Early works attempt to improve the OOD uncertainty estimation by proposing the ODIN score~\citep{liang2018enhancing} and Mahalanobis
distance-based confidence score~\citep{Maha}.
Recent work by Liu \emph{et al.}~\citep{liu2020energy} proposed using an energy score for OOD detection, which demonstrated advantages over the softmax confidence score both empirically and theoretically. Huang and Li~\cite{huang2021mos} proposed a group-based OOD detection method that scales effectively to large-scale dataset ImageNet. Recent work by Lin \emph{et al.}~\cite{lin2021mood} also proposed dynamic OOD inference framework that improved
the computational efficiency of OOD detection. However, previous methods primarily focused on convention non-spurious OOD. We introduce a new formalization of OOD detection that encapsulates both spurious and non-spurious OOD data.


A parallel line of approaches resorts to generative models~\cite{goodfellow2014generative, kingma2018glow} that directly estimate in-distribution density~\citep{nalisnick2019deep, ren2019likelihood, serra2019input, xiao2020likelihood, kirichenko2020normalizing}. In particular, \citet{ren2019likelihood} addressed distinguishing between background and semantic content under unsupervised generative models. Generative approaches yield limiting performance compared with supervised discriminative models due to the lack of label information and typically suffer from high computational complexity.\@
Notably, none of the previous works systematically investigate the influence of spurious correlation for OOD detection.
Our work presents a novel perspective for defining OOD data and investigates the impact of spurious correlation in the training set. Moreover, our formulation is more general and broader than the image background (for example, gender bias in our CelebA experiments is another type of contextual bias beyond image background).

\paragraph{Near-ID Evaluations.} Our proposed spurious OOD can be viewed as a form of near-ID evaluation. Orthogonal to our work, previous works~\citep{winkens2020contrastive, roy2021does} considered the near-ID cases where the \emph{semantics} of OOD inputs are similar to that of ID data (\emph{e.g.}, CIFAR-10 vs. CIFAR-100). In our setting, spurious OOD inputs may have very different semantic labels but are statistically close to the ID data due to shared environmental features (\emph{e.g.}, boat vs. waterbird in Figure 1). While other works have considered domain shift~\citep{GODIN} or covariate shift~\citep{ovadia2019can}, they are more relevant for evaluating model generalization and robustness performance---in which case the goal is to make the model classify accurately into the ID classes and should not be confused with OOD detection task. We emphasize that semantic label shift (i.e., change of invariant feature) is more akin to OOD detection task, which concerns model reliability and detection of shifts where the inputs have disjoint labels from ID data and therefore {should not be predicted by the model}.

\paragraph{Out-of-distribution Generalization.}
Recently, various works have been proposed to tackle the issue of domain generalization, 
which aims to achieve high classification accuracy on new test environments consisting of inputs \emph{with invariant features}, and does not consider the change of invariant features at test time  (i.e., label space $\mathcal{Y}$ remains the same)---a key difference from our focus. Literature in OOD detection is commonly concerned about model reliability and detection of shifts where the OOD inputs have disjoint labels and therefore should not be predicted by the model.  In other words, we consider samples \emph{without invariant features}, regardless of the presence of environmental features or not. 

A plethora of algorithms are proposed: learning invariant representation across domains \citep{ganin2016domain, li2018deep, sun2016deep, li2018domain}, minimizing the weighted combination of risks from training domains \citep{sagawa2019distributionally}, using different risk penalty terms to facilitate invariance prediction \citep{arjovsky2019invariant, krueger2020out}, causal inference approaches \citep{peters2016causal}, and forcing the learned representation different from a set of pre-defined biased representations \citep{bahng2020learning}, mixup-based approaches \citep{zhang2018mixup, wang2020heterogeneous, luo2020generalizing}, etc.  
A recent study \citep{gulrajani2021in} shows that no domain generalization methods achieve superior performance than ERM across a broad range of datasets. 



\paragraph{Contextual Bias in Recognition.}  There has been a rich literature studying the classification performance in the presence of contextual bias~\citep{torralba2003contextual, beery2018recognition, barbu2019objectnet}. The reliance on contextual bias such as image backgrounds, texture, and color for object detection are investigated in \citep{ijcai2017zhu, dcngos2018, geirhos2018imagenettrained, zech2018variable, xiao2021noise, sagawa2019distributionally}. However, the contextual bias for OOD detection is underexplored. In contrast, our study systematically investigates the impact of spurious correlation on OOD detection and how to mitigate it. 

\section{Conclusion}
\label{sec:conclusion}
Out-of-distribution detection is an essential task in open-world machine learning. However, the precise definition is often left in vagueness, and common evaluation schemes can be too primitive to capture the nuances of the problem in reality. In this paper, we present a new formalization where we model the data distributional shifts by considering the invariant and non-invariant features. Under such formalization, we systematically investigate the impact of spurious correlation in the training set on OOD detection and further show insights on detection methods that are more effective in mitigating the impact of spurious correlation. Moreover, we provide theoretical analysis on why reliance on environmental features leads to high OOD detection error. We hope that our work will inspire future research on the understanding and formalization of OOD samples, new evaluation schemes of OOD detection methods, and algorithmic solutions in the presence of spurious correlation. 

\section*{Acknowledgement}
 Research is supported by the Office of the Vice Chancellor for Research and Graduate Education (OVCRGE) with funding from the Wisconsin Alumni Research Foundation (WARF).

\newpage
\bibliography{references}
\bibliographystyle{plainnat}
\newpage


\appendix
\onecolumn
\begin{center}
    \Large{\textbf{Supplementary Material}}
\end{center}

\section{Proofs for Theoretical Results}

\begin{lemmaa} 
\label{lem:bayes}
(Bayes optimal classifier) For any feature vector which is a linear combination of the invariant and environmental features $\Phi_e(\*x) = M_\text{inv}\*z_\text{inv} + M_e\*z_e$, the optimal linear classifier for an environment $e$ has the corresponding coefficient $ 2\Sigma_\Phi^{-1}\bm\mu_{\Phi}$, where:
\begin{align*}
    \bm\mu_{\Phi}&  = M_\text{inv} \bm\mu_{\text{inv}}+M_e \bm\mu_{e}\\
    \Sigma_{\Phi}& = M_\text{inv} M_\text{inv}^{T} \sigma_{\text{inv}}^{2}+M_e M_e^{T} \sigma_{e}^{2}
\end{align*}
\label{lemma:bayes}
\vspace{-20pt}
\end{lemmaa}

\begin{proof}
Since the feature vector $\Phi_e(\*x) = M_\text{inv}\*z_\text{inv} + M_e\*z_e$ is a linear combination of two independent Gaussian densities, $\Phi_e(\*x)$ is also Gaussian with the following density:
\begin{align}
     M_\text{inv}\*z_\text{inv} + M_e\*z_e \mid y \sim \mathcal{N}(y\cdot \underbrace{(M_\text{inv} \bm\mu_{\text{inv}}+M_e \bm\mu_{e})}_{ \bm\mu_{\Phi}}, \underbrace{M_\text{inv} M_\text{inv}^{T} \sigma_{\text{inv}}^{2}+M_e M_e^{T} \sigma_{e}^{2}}_{\Sigma_{\Phi}}).
\end{align}
The conditional density is given by:
\begin{align}
    p(\Phi_e(\*x)  = \phi \mid y) = \frac{1}{\sqrt{ (2\pi)^d |\Sigma_{\Phi}|}} \exp (-\frac{1}{2}(\phi - y\cdot \bm\mu_{\Phi})^\top \Sigma_{\Phi}^{-1} (\phi - y\cdot \bm\mu_{\Phi}))
\end{align}
Then, the probability of $y=1$ conditioned on $\Phi_e(\*x)=\phi$ can be expressed as:
\begin{equation*}
\begin{aligned}
 p\left(y=1 \mid \Phi_e = \phi \right) & =  \frac{ p(\Phi_e = \phi \mid y=1) p(y=1)}{p(\Phi_e = \phi \mid y=1) p(y=1) + p(\Phi_e = \phi \mid y=-1) p(y=-1)}  \\&=\frac{1}{1+\frac{p\left(\Phi_e = \phi \mid y=-1\right)p(y=-1)}{p\left(\Phi_e = \phi \mid y=1\right)p(y=1)}} \\
&=\frac{1}{1+\exp (-y \cdot 2 \phi^{\top} \Sigma_\Phi^{-1}\mu_{\Phi} - \log \eta / (1-\eta) )} \\
&=\sigma\left(y \cdot 2 \phi^{\top} \Sigma_\Phi^{-1}\mu_{\Phi}  + \log \eta / (1-\eta) \right),
\end{aligned}
\end{equation*}
where $\sigma(\cdot)$ is the sigmoid function.  The log odds of $y$ are linear w.r.t. the feature representation $\Phi_e$. Thus given feature $ \left[\begin{array}{c}\Phi_e(\*x)\\ 1\end{array}\right] = \left[\begin{array}{c}\phi \\ 1\end{array}\right]$ (appended with constant 1), the optimal classifier weights are $\left[\begin{array}{c}2\Sigma_\Phi^{-1}\mu_{\Phi} \\ \log \eta / (1-\eta)\end{array}\right]$. Note that the Bayes optimal classifier uses environmental features which are informative of the label but non-invariant. 
\end{proof}

\begin{lemmaa} (Invariant classifier using non-invariant features) 
Suppose $E \leq d_{e}$, given a set of environments $\mathcal{E}=\left\{e_{1}, e_{2}, \ldots, e_{E}\right\}$ such that all environmental means are linearly independent. Then there always exists a unit-norm vector $\*p$ and positive fixed scalar $\beta$ such that $\beta = \*p^{\top} \bm\mu_{e}/ \sigma_{e}^{2}$ $\forall e \in \mathcal{E}$. The resulting optimal classifier weights are 
$$\hat w = \left[ \begin{array}{c} \beta_{\text{inv}} \\ 2\beta \end{array}\right] = \left[ \begin{array}{c} 2\bm\mu_{\text{inv}} / \sigma_\text{inv}^{2}\\ 2\*p^\top\bm\mu_e / \sigma_e^2 \end{array}\right].$$
\label{short_cut}
\vspace{-12pt}
\end{lemmaa}

\begin{proof}
Suppose $M_{\text{inv}} = \left[\begin{array}{l}I_{s\times s} \\ 0_{1\times s}\end{array}\right]$, and $M_e = \left[\begin{array}{c}0_{s\times e} \\ \*p^\top\end{array}\right] $ for some unit-norm vector $\*p \in \mathbb{R}^{d_e}$, then $\Phi_e(\*x) = \left[\begin{array}{c} \*z_\text{inv} \\ \*p^\top\*z_e\end{array}\right] $. By plugging into the results of Lemma~\ref{lem:bayes}, we can obtain the optimal classifier weights as $\left[ \begin{array}{c} 2\mu_{\text{inv}} / \sigma_\text{inv}^{2}\\ 2\*p^\top\bm\mu_e / \sigma_e^2 \end{array}\right]$.\footnote{The constant term is  $\log \eta / (1-\eta)$, as in Proposition~\ref{prop:invariant}.}
If the total number of environments is insufficient (i.e., $E \leq d_E$, which is a practical consideration because datasets with diverse environmental features w.r.t. a specific class of interest are often very computationally expensive to obtain), a short-cut direction $\*p$ that yields invariant classifier weights satisfies the system of linear equations $ A\*p = \*b$, where $A = \left[\begin{array}{l} \bm\mu_1^\top  \\ \cdots \\\bm\mu_E^\top\end{array}\right]$, and $\*b =\left[\begin{array}{l} \sigma_1^2 \\ \cdots \\\sigma_E^2\end{array}\right] $. As $A$ has linearly independent rows and $E \leq d_{e}$, there always exists feasible solutions, among which the minimum-norm solution is given by $\*p = A^\top(AA^\top)^{-1}\*b$. Thus $\beta = 1/\|A^\top(AA^\top)^{-1}\*b\|_2.$
\end{proof}

\begin{theoremm} (Failure of OOD detection under invariant classifier) Consider an out-of-distribution input which contains the environmental feature: $\Phi_{\text{out}}(\*x) = M_{\text{inv}}\*z_{\text{out}} + M_e\*z_e$, 
 where  $\*z_{\text{out}} \perp \bm\mu_{\text{inv}}$. 
Given the invariant classifier (cf. Lemma 2), the posterior probability for the OOD input is $p(y = 1 \mid \Phi_{\text{out}}) = \sigma\left(2\*p^\top\*z_e\beta + \log \eta / (1-\eta) \right)$, where $\sigma$ is the logistic function. Thus for arbitrary confidence $0<c := P(y = 1 \mid \Phi_{\text{out}})<1$, there exists $\Phi_{\text{out}}(\*x)$ with $\*z_e$ such that $\*p^\top\*z_e = \frac{1}{2\beta} \log \frac{c(1 -\eta)}{\eta(1-c)}$. 
\end{theoremm}

\begin{proof}
Consider an out-of-distribution input $\*x_{\text{out}}$ with  $ M_\text{inv} = \left[\begin{array}{l}I_{s\times s} \\ 0_{1\times s}\end{array}\right]$, and $M_e = \left[\begin{array}{c}0_{s\times e} \\ \*p^\top\end{array}\right] $, then the feature representation is $\Phi_e(\*x) = \left[\begin{array}{c} \*z_\text{out} \\ \*p^\top\*z_e\end{array}\right]$, where $\*p$ is the unit-norm vector defined in Lemma~\ref{lem:inv}.
By Bayes' rule, the posterior probability of $y=1$ can be expressed as:
\begin{equation}
\label{eq:general}
\begin{aligned}
P\left(y = 1 \mid \*z_{\text{out}}, \*p^\top\*z_e\right) &=\frac{P\left(\*z_{\text{out}},\*p^\top\*z_e, y = 1\right)}{P\left(\*z_{\text{out}}, \*p^\top\*z_e\right)} \\
&=\frac{P\left(\*z_{\text{out}} \mid y=1\right) P(\*p^\top\*z_e \mid y=1)P(y=1)}{P\left(\*z_{\text{out}}, \*p^\top\*z_e\right)}\\
&=\frac{1}{1 + \frac{P\left(\*z_{\text{out}} \mid y=-1\right) P(\*p^\top\*z_e \mid y=-1)P(y=-1)}{P\left(\*z_{\text{out}} \mid y=1\right) P(\*p^\top\*z_e \mid y=1)P(y=1)} }
\end{aligned}
\end{equation}

Recall that the conditional density is given by:
\begin{align}
    p(\*z_\text{out} \mid y) = \frac{1}{\sqrt{ (2\pi)^s |\sigma_\text{inv}^2 I|}} \exp (-\frac{1}{2}(\*z_\text{out}  - y\cdot \bm\mu_\text{inv})^\top \frac{1}{\sigma^2_\text{inv} } \cdot I \cdot (\*z_\text{out}  - y\cdot \bm\mu_\text{inv})).
\end{align}

Canceling common terms, we get
\begin{equation}
\label{eq:res}
\begin{aligned}
P\left(y = 1 \mid \*z_{\text{out}}, \*p^\top\*z_e\right) 
&=\frac{1}{1+ \frac{\exp\left( -  \bm\mu_\text{inv}^\top \*z_{\text{out}}/\sigma_\text{inv}^2 - \*p^\top\*z_e \beta\right)(1-\eta) }{\exp\left(\bm\mu_\text{inv}^\top \*z_{\text{out}}/\sigma_\text{inv}^2 + \*p^\top\*z_e \beta\right) \eta}} \\
&=\frac{1}{1+\exp \left(-\left(2 \*p^\top\*z_e \beta + \log \eta / (1-\eta)\right) \right)}
\end{aligned}
\end{equation}
Then we have $P(y = 1 \mid \Phi_{\text{out}})  = P(y = 1 \mid \*z_{\text{out}}, \*p^\top\*z_e) = \sigma\left(2\*p^\top\*z_e\beta + \log \eta / (1-\eta) \right)$, where $\sigma$ is the logistic function. Thus for arbitrary confidence $0<c := P(y = 1 \mid \Phi_{\text{out}})<1$, there exists $\Phi_{\text{out}}(\*x)$ with $\*z_e$ such that $\*p^\top\*z_e = \frac{1}{2\beta} \log \frac{c(1 -\eta)}{\eta(1-c)}$.
\end{proof}

\textbf{Remark: }  In a more general case, $\*z_{\text{out}}$ can be modeled as a random vector  that is independent of the in-distribution labels $y=1$ and $y=-1$ and environmental features: $\*z_{\text{out}} \indep y$ and $\*z_{\text{out}} \indep \*z_e$. Thus in Eq.~\ref{eq:general} we have $P\left(\*z_{\text{out}} \mid y =1\right) = P\left(\*z_{\text{out}} \mid y = -1\right) = P\left(\*z_{\text{out}}\right)$. Then $P(y = 1 \mid \Phi_{\text{out}})  = \sigma\left(2\*p^\top\*z_e\beta + \log \eta / (1-\eta) \right)$, same as in Eq.~\ref{eq:res}. Therefore our main theorem still holds under more general case. 

\newpage
\section{Extension: Color Spurious Correlation}
To further validate our findings beyond background and gender spurious (environmental) features, we provide additional experimental results with the ColorMNIST dataset, as shown in Figure~\ref{fig:mnist_illust}. 

\begin{figure*}[ht]
  \centering
  \vspace{-0.3cm}
    \includegraphics[width=0.9\linewidth]{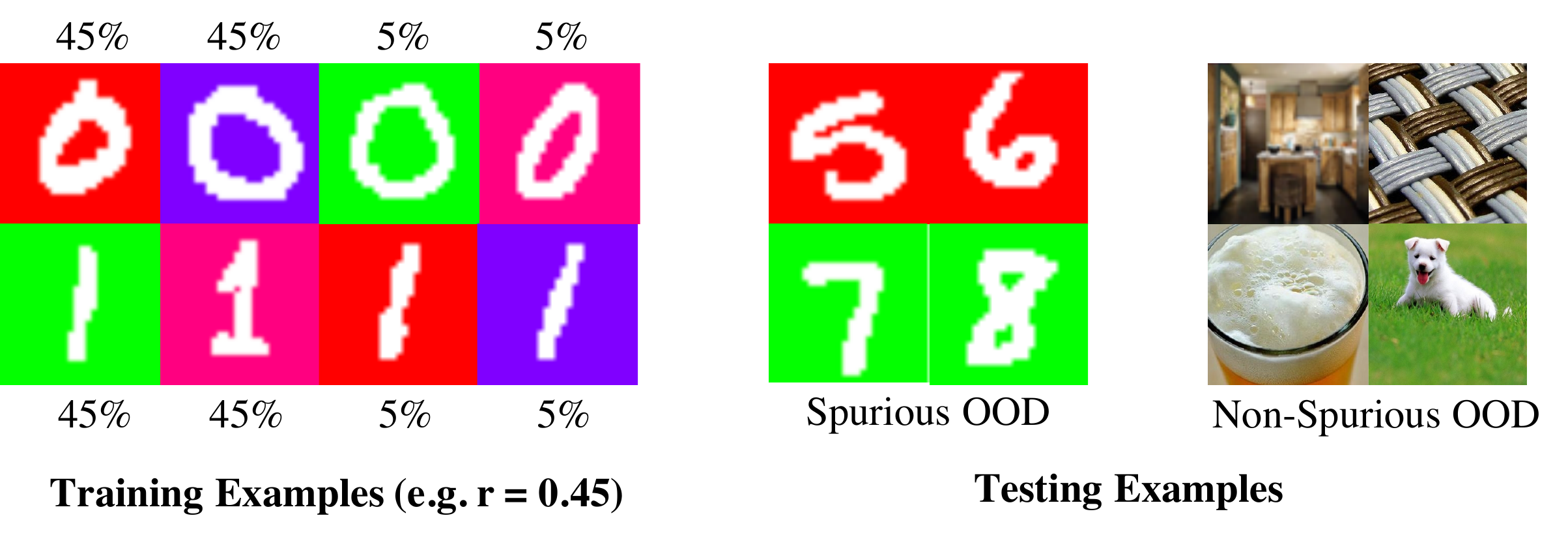}
\caption{\small \textbf{Left}: Training environments of ColorMNIST. The digit 0 correlates both red and purple background with probability $r$, whereas digit 1 correlates with green and pink with probability $r$. \textbf{Right}: Spurious OOD inputs contain the shared environmental feature (color background) yet with different digit labels (\emph{e.g.}, not 0 or 1). Non-spurious (conventional) OOD samples share neither the digit semantics nor colors in the training set.}
\label{fig:mnist_illust}
\end{figure*}

\paragraph{Evaluation Task 3: ColorMNIST.} The ColorMNIST dataset is modified from MNIST~\citep{lecun1998gradient}, which composes colored backgrounds on digit images. In this dataset, $\mathcal{E}=\{\texttt{red}, \texttt{green}, \texttt{purple},\texttt{pink}\}$ denotes the background color and we use $\mathcal{Y}=\{0,1\}$ as in-distribution classes. The correlation between the background color $e$ and the digit $y$ is explicitly controlled, with $r \in \{0.25, 0.35, 0.45\}$. That is, $r$ denotes the probability of $P(e=\texttt{red} \mid y=0) = P(e=\texttt{purple} \mid y=0) = P(e=\texttt{green} \mid y=1) = P(e=\texttt{pink} \mid y=1)$, while $0.5-r = P(e=\texttt{green} \mid y=0) = P(e=\texttt{pink} \mid y=0) = P(e=\texttt{red} \mid y=1) = P(e=\texttt{purple} \mid y=1)$. Note that the maximum correlation $r$ (reported in Table~\ref{tab:erm_ablation_cmnist}) is $0.45$. As ColorMNIST is relatively simpler compared to Waterbirds and CelebA, further increasing the correlation results in less interesting environments where the learner can easily pick up the contextual information.
For spurious OOD, we use digits $\{5,6,7,8,9\}$ with background color \texttt{red} and \texttt{green}, which contain overlapping environmental features as the training data. For non-spurious OOD, following common practice \citep{MSP}, we use the \texttt{Textures}~\citep{cimpoi2014describing}, \texttt{LSUN}~\citep{lsun} and \texttt{iSUN}~\citep{xu2015turkergaze} datasets.  We train on ResNet-18~\citep{he2016deep}, which achieves $99.9\%$ accuracy on the in-distribution test set. The OOD detection performance is shown in  Table~\ref{tab:erm_ablation_cmnist}.

\begin{table*}[!h]
    \ra{1.2}
    \centering
    \resizebox{0.85\textwidth}{!}{
    \begin{tabular}{cccccccc}\toprule
          & & \multicolumn{2}{c}{$\textbf{r=0.25}$} & \multicolumn{2}{c}{$\textbf{r=0.35}$}
          & \multicolumn{2}{c}{$\textbf{r=0.45}$}
          \\\cmidrule(lr){3-4}\cmidrule(lr){5-6}\cmidrule(lr){7-8} \textbf{OOD Type} & \textbf{Test Set}&
              \textbf{FPR95} $\downarrow$     &      \textbf{AUROC} $\uparrow$   &    \textbf{FPR95}  $\downarrow$     &      \textbf{AUROC}$\uparrow$    &
                  \textbf{FPR95} $\downarrow$      &      \textbf{AUROC} $\uparrow$ \\  \midrule
    \textbf{Spurious OOD } &  &$5.40\pm1.81$ &$98.25\pm0.89 $  &$13.5\pm1.90$  &$94.91\pm0.86$ & $30.45\pm10.42$
             & $86.74\pm7.76$ \\
            \midrule
            \multirow{4}*{    \textbf{Non-spurious OOD}}
    &  Texture  &$ 0.03\pm0.03$&$99.42 \pm0.35$ & $0.43\pm0.49$&  $99.41\pm0.17$&$9.93\pm5.26$&$ 95.94\pm0.88$ \\
   &  iSUN  &  $0.18\pm0.19$& $99.43\pm0.27$& $0.25\pm0.10$&  $99.39\pm0.14$&$6.68\pm4.23$ &$98.16\pm1.25$  \\
    &  LSUN  &   $0.3\pm0.28$& $99.55\pm0.17$&  $0.63\pm0.40$&  $99.40\pm0.19$&$6.33\pm4.93$ &$98.51\pm0.80$  \\
    \bottomrule
    \end{tabular}
    }
\caption{\small OOD detection performance of models trained on \textbf{ColorMNIST}. Increased spurious correlation in the training set results in worsen performance for both non-spurious and spurious OOD samples. For any fixed spurious correlation, spurious OOD is more challenging than non-spurious OOD samples. Results (mean and std) are estimated over 4 runs for each setting.}
     \label{tab:erm_ablation_cmnist}
\end{table*}

\section{Visualization and Histograms}
\paragraph{Visualization.} As an extension of Section~\ref{sec:ood_score}, here we present the visualization of embeddings for ID samples and samples from non-spurious OOD test sets LSUN (Figure~\ref{fig:umap_lsun}) and iSUN (Figure~\ref{fig:umap_isun}) based on the CelebA task. We can observe that for both non-spurious OOD test sets, the feature representations of ID and OOD are separable, similar to observations in Section~\ref{sec:ood_score}.

\begin{figure*}[h]
  \centering
  \begin{subfigure}[b]{0.35\linewidth}
    \includegraphics[width=\linewidth]{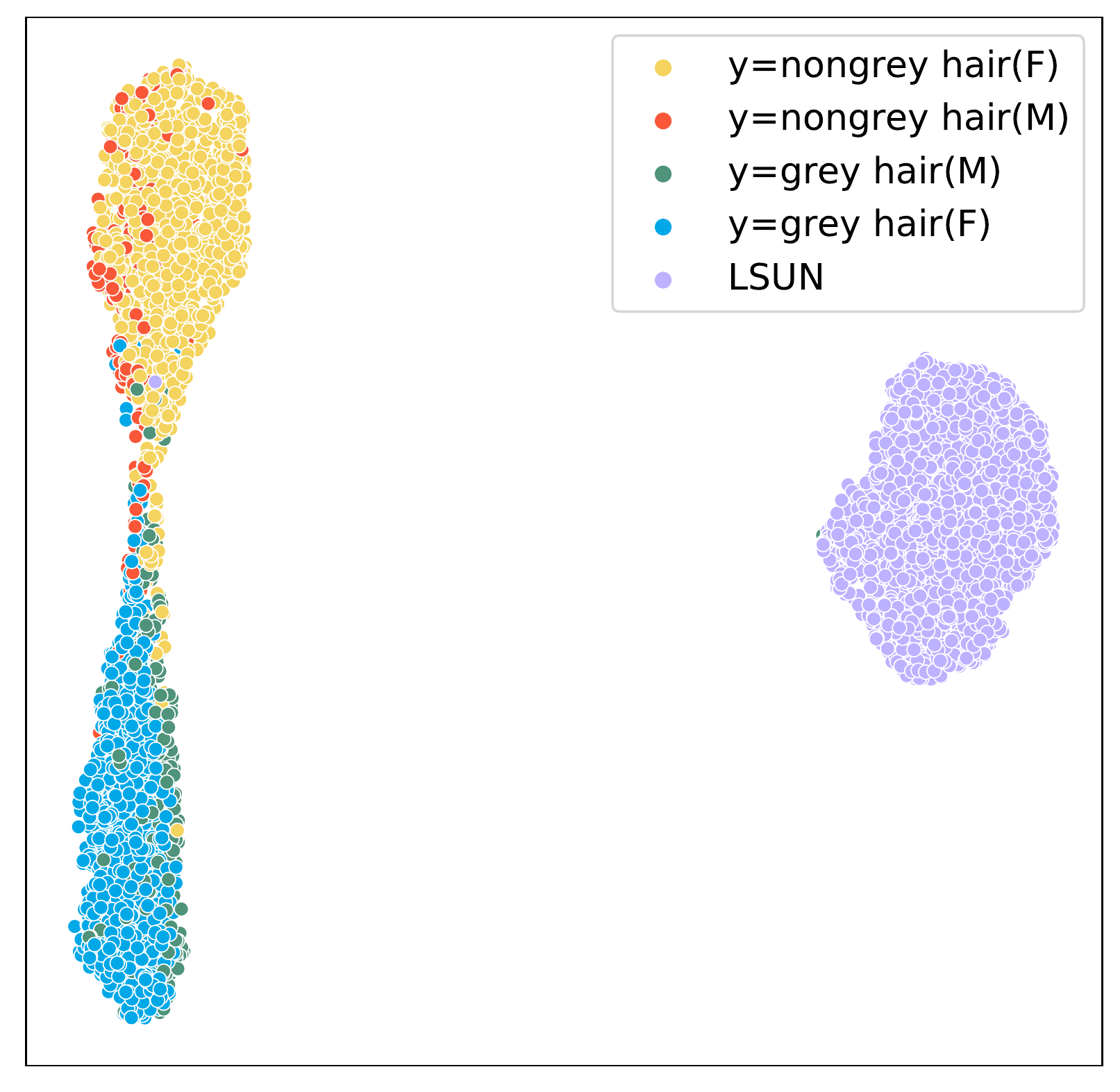}
     \caption{Feature for ID data and LSUN (Non-spurious). }
     \label{fig:umap_lsun}
  \end{subfigure}
  \hspace{4em}
    \begin{subfigure}[b]{0.35\linewidth}
    \includegraphics[width=\linewidth]{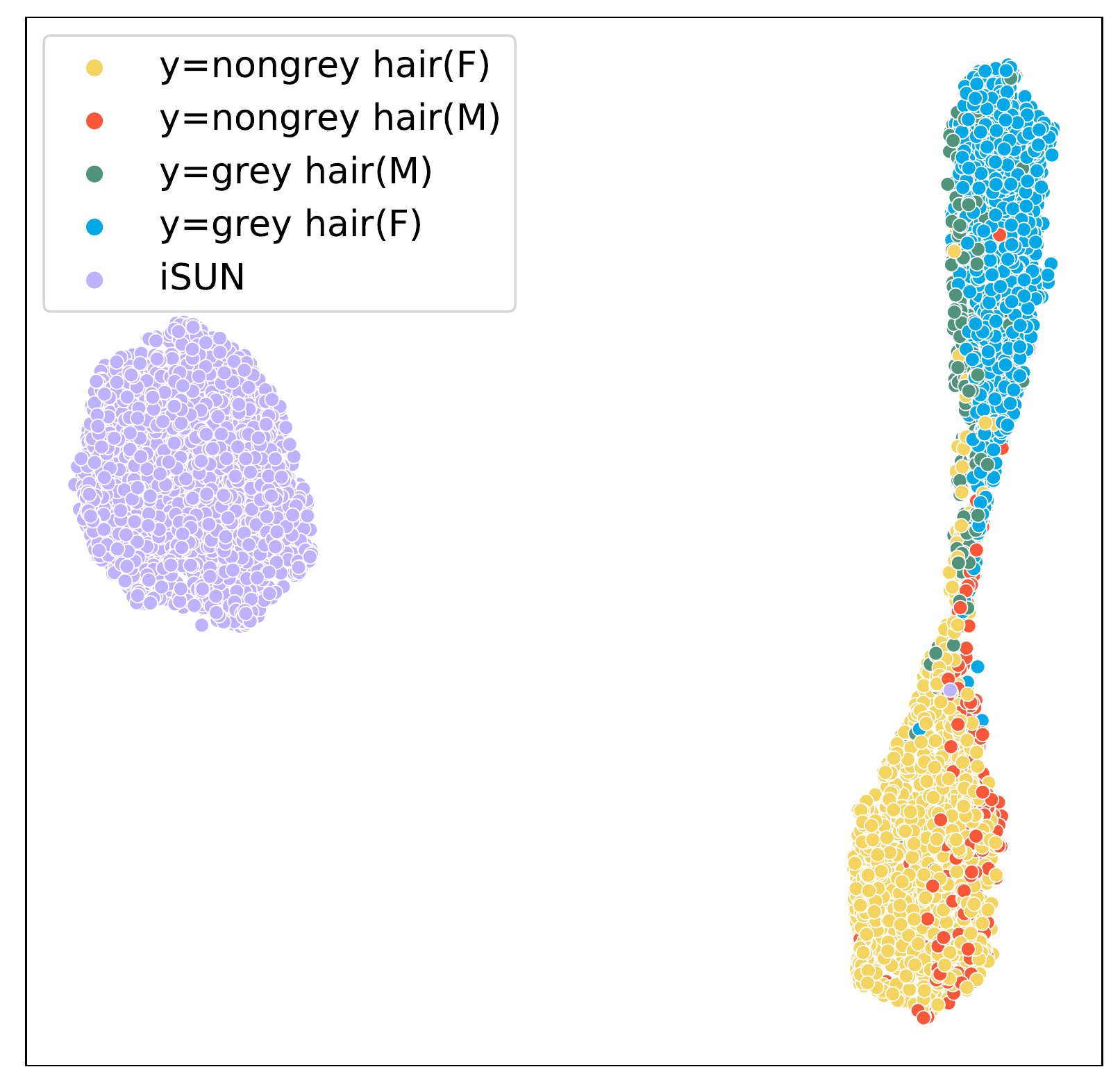}
    \caption{Feature for ID data and iSUN (Non-spurious). }
      \label{fig:umap_isun}
  \end{subfigure}
\caption{ \small Visualization of feature embedding for in-distribution samples and non-spurious OOD samples: LSUN (Left) and iSUN (right).} 
\end{figure*}
\paragraph{Histograms.} We also present histograms of the Mahalanobis distance score and MSP score for non-spurious OOD test sets iSUN and LSUN based on the CelebA task. As shown in Figure~\ref{fig:hist_other}, for both non-spurious OOD datasets, the observations are similar to what we describe in Section~\ref{sec:ood_score} where ID and OOD are more separable with Mahalanobis score than MSP score. This further verifies that feature-based methods such as Mahalanobis score is promising to mitigate the impact of spurious correlation in the training set for non-spurious OOD test sets compared to output-based methods such as MSP score.
\begin{figure*}[!h]
  \centering
  \vspace{-0.3cm}
    \includegraphics[width=0.75\linewidth]{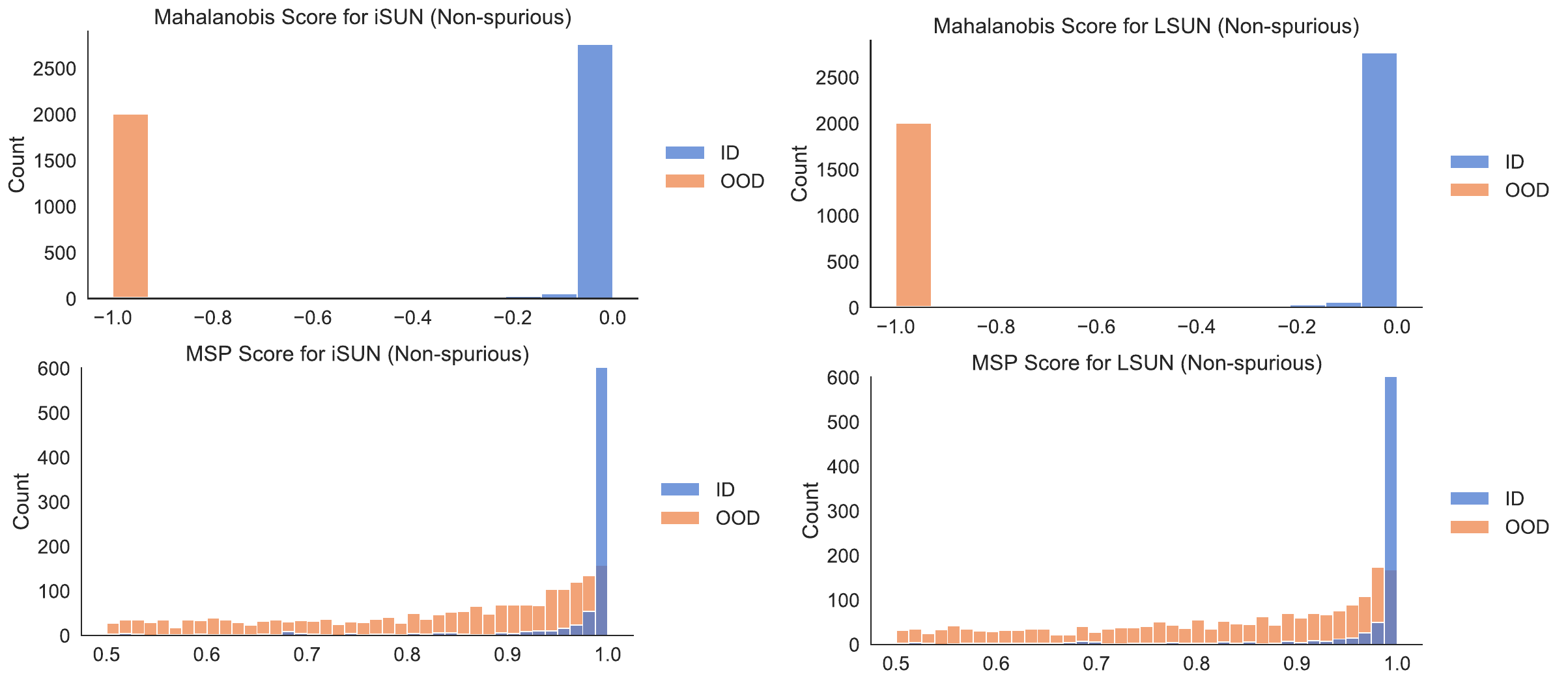}
\caption{\small \textbf{Left}: Histograms of the Mahalanobis score and MSP score for iSUN (Non-spurious OOD). \textbf{Right}: Histograms of the Mahalanobis score and MSP score for LSUN (Non-spurious OOD).}
\label{fig:hist_other}
\end{figure*}
\section{Extension: Adjusting Spurious Correlation in the Training Set for CelebA}
To further validate if our observations on the impact of the extent of spurious correlation in the training set still hold beyond the Waterbirds and ColorMNIST tasks, here we subsample the CelebA dataset (described in Section~\ref{sec:erm}) such that the spurious correlation is reduced to $r = 0.7$. Note that we do not further reduce the correlation for CelebA because that will result in a small size of total training samples in each environment which may make the training unstable. The results are shown in Table~\ref{tab:more_celebA}. The observations are similar to what we describe in Section~\ref{sec:erm} where increased spurious correlation in the training set results in worsened performance for both non-spurious and spurious OOD samples. For example, the average FPR95 is reduced by $3.37\%$ for LSUN, and $2.07\%$ for iSUN when $r = 0.7$ compared to $r=0.8$. In particular, spurious OOD is more challenging than non-spurious OOD samples under both spurious correlation settings.
\begin{table*}[!h]
    \ra{1.2}
    \centering
    \scalebox{0.85}{
    \begin{tabular}{cccccc}\toprule
          & & \multicolumn{2}{c}{$\textbf{r=0.7}$}
          & \multicolumn{2}{c}{$\textbf{r=0.8}$}
          \\\cmidrule(lr){3-4}\cmidrule(lr){5-6} \textbf{OOD Type} & \textbf{Test Set}&
              \textbf{FPR95} $\downarrow$     &      \textbf{AUROC} $\uparrow$   &    \textbf{FPR95}  $\downarrow$     &      \textbf{AUROC}$\uparrow$    $\uparrow$ \\  \midrule
    \textbf{Spurious OOD} & & $70.18\pm1.76$ & $83.30\pm0.68$&    $71.28\pm4.12$ & $82.04\pm2.64$  \\
            \midrule
            \multirow{4}*{    \textbf{Non-spurious OOD}}
   &  iSUN  & $15.28\pm3.18$ & $97.37\pm0.14$ & $17.35\pm2.97$ &  $97.03\pm0.30$ \\
    &  LSUN  & $15.48\pm3.57$  & $97.56\pm0.20$   & $18.85\pm2.44$ & $96.90\pm0.17$  \\
    &  SVHN  & $5.39\pm 4.30$ & $98.89\pm0.90$ & $5.63\pm2.60$ &  $98.64\pm0.21$  \\
    \bottomrule
    \end{tabular}
    }
    \caption{\small OOD detection performance of  models trained on \textbf{CelebA}~\citep{liu2015faceattributes}. The observations are similar to the Waterbirds and ColorMNIST tasks. Increased spurious correlation in the training set results in worsen performance for both non-spurious and spurious OOD samples. In particular, spurious OOD is more challenging than non-spurious OOD samples.  Results (mean and std) are estimated over 4 runs for each setting.}
    \label{tab:more_celebA}
     \vspace{-0.3cm}
\end{table*}

\section{Extension: Training with Domain Invariance Objectives}
\label{sec:irm}

In this section, we provide empirical validation of our analysis in Section~\ref{sec:theory}, where we evaluate the OOD detection performance based on models that are trained with recent prominent domain invariance learning objectives
where the goal is to find a classifier that does not overfit to environment-specific properties of the data distribution.
Note that OOD \textbf{generalization} aims to achieve high classification accuracy on new test environments consisting of inputs \emph{with invariant features}, and does not consider the absence of invariant features at test time---a key difference from our focus. In the setting of spurious OOD \textbf{detection}, we consider test samples in environments \emph{without invariant features}. 
We begin by describing the more popular objectives and include a more expansive list of invariant learning approaches in our study. 

\paragraph{Invariant Risk Minimization (IRM).} IRM~\citep{arjovsky2019invariant} assumes the existence of a feature representation $\Phi$ such that the optimal classifier on top of these features is the same across all environments. To learn this $\Phi$, the IRM objective solves the following bi-level optimization problem:
\begin{equation}
\label{irm_obj}
\min _{\Phi, \hat{w}} \frac{1}{|\mathcal{E}|} \sum_{e \in \mathcal{E}} \mathcal{R}^{e}(\Phi, \hat{w}) \quad \text { s.t. } \quad \hat{w} \in \underset{w}{\arg\min } \mathcal{R}^{e}(\Phi, w) \quad \forall e \in \mathcal{E}
\end{equation}

The authors also propose a practical version named IRMv1 as a surrogate to the original challenging bi-level optimization formula (\ref{irm_obj}) which we adopt in our implementation:
\begin{equation}
    \min _{\Phi: \mathcal{X} \rightarrow \mathcal{Y}} \sum_{e \in \mathcal{E}_{\mathrm{tr}}} R^{e}(\Phi)+\lambda \cdot\left\|\nabla_{w \mid w=1.0} R^{e}(w \cdot \Phi)\right\|^{2}
\end{equation}
where an empirical approximation of the gradient norms in IRMv1 can be obtained by a balanced partition of batches from each training environment.

\paragraph{Group Distributionally Robust Optimization (GDRO).} GDRO~\citep{sagawa2019distributionally} minimizes the worst-group risk:
\begin{equation}
\label{gdro}
    \min_{w} \max_{g\in \mathcal{G}} \mathbb{E}_{(\*x,y)\sim \hat P_g}[\ell(w;(x,y))],
\end{equation}
where each example belongs to a group $g\in \mathcal{G} = \mathcal{Y}\times \mathcal{E}$, with $g=(y,e)$.
The model learns the correlation between label $y$ and environment $e$ in the training data would do poorly on minority group where the correlation does not hold. Hence, by minimizing the worst-group risk, the model is discouraged from relying on spurious features. The authors show that objective (\ref{gdro}) can be rewritten as:
\begin{equation}
    \min_{w} \sup_{q \in \Delta{m}} \sum_{g=1}^{m} q_{g} \mathbb{E}_{(x, y) \sim P_{g}}[\ell(w ;(x, y))]
\end{equation}
Then Algorithm 1 in \citep{sagawa2019distributionally} can be used for optimization where stochastic gradient descent on $w$ is interleaved with exponentiated gradient ascent on $q$. For further details and convergence analysis, we encourage interested readers to refer to \citep{sagawa2019distributionally}.

\paragraph{Alternative Objectives.} IRM is motivated by the existence of a feature representation $\Phi$ such that $\mathbb{E}[y \vert \Phi(\*x)]$ is invariant across environments. Follow-up works proposed several variations, based on different notions of invariance. In particular, ~\citep{krueger2020out} proposed Risk Extrapolation (\textbf{REx}), which aims to achieve stronger invariance $p(y \vert \Phi(\*x))$ by penalizing the variance of risks of environments. Other approaches have proposed to remove the predictability of $p(e\vert \Phi(\*x))$ through domain adversarial losses such as \textbf{DANN}~\citep{ganin2016domain} and \textbf{CDANN}~\citep{ li2018deep} (adapted for domain generalization). 
For completeness, we include all the aforementioned methods in our study\footnote{Our implementation for most of the training objects are based on: \url{https://github.com/facebookresearch/DomainBed}.}. 

\begin{table}[h]
    \ra{1.2}
    \centering
    \small
    \scalebox{1}{
    \begin{tabular}{lcc}\toprule
         \textbf{Training Objective} &       \textbf{FPR95} $\downarrow$     &      \textbf{AUROC} $\uparrow$  \\\midrule
    \textbf{ERM}~\citep{vapnik1992principles}  & $71.28\pm4.12$ & $82.04\pm2.64$  \\
    \textbf{IRM}~\citep{arjovsky2019invariant}   & $70.09\pm3.67$ & $82.66\pm3.28$ \\
    \textbf{GDRO}~\citep{sagawa2019distributionally}  & $68.77\pm4.56$ & $83.39\pm3.12$ \\
    \textbf{REx}~\citep{krueger2020out}   & $72.43\pm3.21$ & $81.88\pm3.19$ \\
    \textbf{DANN}~\citep{ganin2016domain}  & $70.01\pm7.47$ & $82.30\pm9.89$ \\
    \textbf{CDANN}~\citep{li2018deep} & $69.87\pm4.19$ & $82.93\pm4.55$ \\
    \bottomrule
    \end{tabular}
    }
    \caption{\small Spurious OOD detection performance on CelebA~\citep{liu2015faceattributes} where $r\approx0.8$. The models are trained with domain invariance learning objectives. The results verify that detecting spurious OOD data is challenging as no training objectives significantly outperform ERM.}
    \label{tab:invariance}
\end{table}

\begin{table}[h]
    \ra{1.2}
    \centering
    \small
    \scalebox{1}{
    \begin{tabular}{lcc}\toprule
         \textbf{Training Objective} &       \textbf{FPR95} $\downarrow$     &      \textbf{AUROC} $\uparrow$  \\\midrule
    \textbf{ERM}~\citep{vapnik1992principles}  & $74.22\pm13.12$ & $80.98\pm4.45$  \\
    \textbf{IRM}~\citep{arjovsky2019invariant}   & $72.41\pm13.27$ & $81.29\pm5.24$ \\
    \textbf{GDRO}~\citep{sagawa2019distributionally}  & $70.79\pm11.51$ & $82.94\pm4.59$ \\
    \textbf{REx}~\citep{krueger2020out}   & $73.83\pm15.26$ & $81.25\pm4.99$ \\
    \textbf{DANN}~\citep{ganin2016domain}  & $72.81\pm13.47$ & $81.11\pm6.21$ \\
    \textbf{CDANN}~\citep{li2018deep} & $72.37\pm14.20$ & $82.13\pm3.53$ \\
    \bottomrule
    \end{tabular}
    }
    \caption{\small Spurious OOD detection performance on Waterbirds~\citep{sagawa2019distributionally} where $r=0.7$. The models are trained with domain invariance learning objectives. The results are similar to what we observe for CelebA, where detecting spurious OOD data is challenging.}
    \label{tab:invariance_waterbird}
\end{table}

\paragraph{Results.} Table~\ref{tab:invariance} summarizes the OOD detection performance for spurious OOD samples based on models trained with various invariance learning objectives. All methods are trained on the CelebA dataset described in Section~\ref{sec:erm} where ``Grey hair'' is highly correlated with ``Male' in the training set ($r\approx0.8$). We then compute the energy score~\citep{liu2020energy} from the model output $f(\*x)$ as OOD uncertainty measurement for OOD detection. 
From the table, we can observe that despite being motivated by invariance learning, many objectives do not significantly outperform the ERM baseline. For example, DGRO only mildly improves over ERM ($1.35\%$ improvement in terms of AUROC). Moreover, invariance learning methods generally display larger variances across runs compared to ERM. 
Similar observations still hold for Waterbirds, where we choose $r=0.7$, as shown in Table~\ref{tab:invariance_waterbird}. A recent study \citep{gulrajani2021in} shows that ERM remains competitive in OOD generalization tasks compared with various domain invariance learning methods across a broad range of real-world datasets. While our results suggest that given high spurious correlation in the training set, detecting spurious OOD remains challenging, even for models trained with domain invariance objectives.

\section{Experiment Details and In-distribution Classification Performance}
\paragraph{Software and Hardware.}
Our code is implemented with Python 3.8.0 and PyTorch 1.6.0.  All experiments are run on NVIDIA GeForce RTX 2080Ti.
\paragraph{Experiment Details.}
Following the common setup, the validation set is randomly selected from 20\% of the training set. We perform grid search over learning rate $\gamma \in \{0.0001, 0.005, 0.001, 0.01\}$ and $l_2$ penalties $\lambda \in \{0.0001, 0.001, 0.01, 0.05\}$. We train for 30 epochs with SGD on ResNet-18. For ColorMNIST, we train from scratch while we start training with pre-trained ResNet for Waterbirds and CelebA, as in~\citet{sagawa2019distributionally}.
\paragraph{In-distribution Classification Performance.}
Table~\ref{tab:color_mnist_precision}, Table~\ref{tab:waterbird_precision}, and Table~\ref{tab:celebA_precision} present the in-distribution data classification accuracy for models trained with ERM and other domain invariance learning objectives for different tasks respectively (averaged over 4 runs).
\begin{table}[h]
    \ra{1.2}
    \centering
    \scalebox{0.88}{
    \begin{tabular}{lccc}\toprule
      \textbf{Training Objective}& $\textbf{r=0.25}$ & $\textbf{r=0.35}$ & $\textbf{r=0.45}$ \\
    \midrule
    \textbf{ERM}~\citep{vapnik1992principles}   &  $99.98\pm0.03$ & $99.99\pm0.02$ & $99.97\pm0.03$ \\
    \textbf{IRM}~\citep{arjovsky2019invariant}  & $99.98\pm0.02$ & $100.00\pm0.00$ & $99.99\pm0.02$ \\
    \textbf{GDRO}~\citep{sagawa2019distributionally}  & $99.97\pm0.04$ & $99.98\pm0.03$ & $99.98\pm0.02$ \\
    \textbf{REx}~\citep{krueger2020out}   & $100.00\pm0.00$ & $99.99\pm0.02$ & $99.99\pm0.02$ \\
    \textbf{DANN}~\citep{ganin2016domain}  & $99.97\pm0.02$ & $99.99\pm0.02$ & $99.99\pm0.02$ \\
    \textbf{CDANN}~\citep{li2018deep} & $99.97\pm0.02$ & $99.99\pm0.02$ & $99.98\pm0.02$  \\
    \bottomrule
    \end{tabular}
    }
    \caption{\small In-Distribution data classification accuracy on ColorMNIST.}
    \label{tab:color_mnist_precision}
\end{table}
\begin{table}[H]
    \ra{1.2}
    \centering
    \scalebox{0.88}{
    \begin{tabular}{lcccc}\toprule
    \textbf{Training Objective} & $\textbf{r=0.5}$ & $\textbf{r=0.7}$ & $\textbf{r=0.9}$ \\
    \midrule
    \textbf{ERM}~\citep{vapnik1992principles}   &  $96.93\pm0.05$ & $96.64\pm0.06$ & $94.67\pm0.25$ \\
    \textbf{IRM}~\citep{arjovsky2019invariant}  & $96.48\pm0.24$ & $96.67\pm0.13$ & $94.70\pm0.44$ \\
    \textbf{GDRO}~\citep{sagawa2019distributionally}  & $96.82\pm0.00$ & $96.63\pm0.04$ & $94.55\pm0.12$ \\
    \textbf{REx}~\citep{krueger2020out}   & $97.11\pm0.07$ & $96.67\pm0.14$ & $94.68\pm0.10$ \\
    \textbf{DANN}~\citep{ganin2016domain}  & $96.65\pm0.12$ & $96.08\pm0.08$ & $93.57\pm0.48$\\
    \textbf{CDANN}~\citep{li2018deep} & $96.57\pm0.09$ & $96.19\pm0.13$ & $94.17\pm0.21$ \\
    \bottomrule
    \end{tabular}
    }
    \caption{\small In-Distribution data classification accuracy on Waterbirds~\citep{sagawa2019distributionally}.}
    \label{tab:waterbird_precision}
\end{table}


\begin{table}[h]
    \ra{1.2}
    \centering
    \small
    \scalebox{1}{
    \begin{tabular}{lcc}\toprule
         \textbf{Training Objective}  & \textbf{$\textbf{r=0.8}$} \\\midrule
    \textbf{ERM}~\citep{vapnik1992principles}  & $95.78\pm0.48$ \\
    \textbf{IRM}~\citep{arjovsky2019invariant}   & $95.97\pm0.62$ \\
    \textbf{GDRO}~\citep{sagawa2019distributionally} & $95.74\pm0.54$ \\
    \textbf{REx}~\citep{krueger2020out} & $95.49\pm0.77$ \\
    \textbf{DANN}~\citep{ganin2016domain} & $96.27\pm0.25$ \\
    \textbf{CDANN}~\citep{li2018deep} & $94.74\pm0.63$ \\
    \bottomrule
    \end{tabular}
    }
    \caption{\small In-Distribution data classification accuracy on CelebA~\citep{liu2015faceattributes}. }
    \vspace{-0.4cm}
    \label{tab:celebA_precision}
\end{table}

\end{document}